\documentclass[english,american]{IEEEtran}
\usepackage[T1]{fontenc}
\usepackage[utf8]{inputenc}
\usepackage{array}
\usepackage{float}
\usepackage{bm}
\usepackage{amsmath}
\usepackage{amsthm}
\usepackage{amssymb}
\usepackage{graphicx}

\makeatletter

\providecommand{\tabularnewline}{\\}
\floatstyle{ruled}
\newfloat{algorithm}{tbp}{loa}
\providecommand{\algorithmname}{Algorithm}
\floatname{algorithm}{\protect\algorithmname}

\theoremstyle{plain}
\newtheorem{thm}{\protect\theoremname}
\theoremstyle{remark}
\newtheorem{rem}[thm]{\protect\remarkname}
\theoremstyle{plain}
\newtheorem{lem}[thm]{\protect\lemmaname}

\usepackage{color}
\usepackage{psfrag}\usepackage{cite}\usepackage{subfigure}

\usepackage{flushend}

\author{

\IEEEauthorblockN{Zheng~Xing and Junting~Chen}

\IEEEauthorblockA{School of Science and Engineering (SSE) and Shenzhen Future Network of Intelligence Institute (FNii-Shenzhen) \\ The Chinese University of Hong Kong, Shenzhen, Guangdong 518172, China}
}

\usepackage[acronym]{glossaries}
\newcommand{\newac}{\newacronym}
\newcommand{\ac}{\gls}

\newcommand{\acpl}{\glspl}

\makeglossaries

\usepackage{enumitem}
\usepackage{amsmath} 
\usepackage{cuted}   
\newac{ips}{IPS}{indoor positioning system}
\newac{wlan}{WLAN}{wireless local area network}
\newac{ap}{AP}{access point}
\newac{sc}{SC}{subspace clustering}
\newac{uwb}{UWB}{ultrawideband}
\newac{hmm}{HMM}{hidden Markov model}
\newac{ema}{EM}{expectation maximization}
\newac{lbs}{LBS}{location-based service}
\newac{gpc}{GPC}{Gaussian process classification}
\newac{gpca}{GPCA}{generalized PCA}
\newac{msl}{MSL}{multistage learning}
\newac{mppca}{MPPCA}{mixture of probabilistic PCA}
\newac{alc}{ALC}{agglomerative lossy compression}
\newac{ransac}{RANSAC}{random sample consensus}
\newac{fa}{FA}{factorization-based affinity}
\newac{gpcaa}{GPCAA}{gpca-based affinity}
\newac{slbf}{SLBF}{spectral local best-fit flats}
\newac{lsa}{LSA}{local subspace affinity}
\newac{llmc}{LLMC}{locally linear manifold clustering}
\newac{ssc}{SSC}{sparse subspace clustering}
\newac{lrr}{LRR}{low-rank representation}
\newac{scc}{SCC}{spectral curvature clustering}
\newac{mfb}{MFB}{matrix factorization-based}
\newac{mssc}{MSSC}{multi-view low-rank sparse subspace clustering}
\newac{dsc}{DSC}{deep subspace clustering}
\newac{sssc}{SSSC}{Stochastic sparse subspace clustering}
\newac{mdssc}{MDSSC}{}
\newac{nmi}{NMI}{normalized mutual information}
\newac{svm}{SVM}{support vector machine}
\newac{tdoa}{TDOA}{time difference of arrival}
\newac{gmm}{GMM}{Gaussian mixture model}
\newac{pca}{PCA}{Principal Component Analysis}
\newac{dnn}{DNN}{deep neural network}
\newac{mlp}{MLP}{multilayer perceptron}
\newac{admm}{ADMM}{alternating direction method of multipliers}
\newac{gan}{GAN}{generative adversarial networks}
\newac{rsrp}{RSRP}{reference signal received power}
\newac{ss}{SS}{synchronization signal}
\newac{fim}{FIM}{Fisher information matrix}
\newac{crlb}{CRLB}{Cramer-Rao Lower Bound}
\newac{1nn}{1-NN}{1 Nearest-Neighbor}
\newac{dft}{DFT}{discrete Fourier transform}
\newac{evd}{EVD}{Eigen Value Decomposition}
\newac{ssb}{SSB}{secondary synchronization block}
\newac{lstm}{LSTM}{long short-term memory}
\newac{ppp}{PPP}{Poisson Point Process}
\newac{ofdm}{OFDM}{Orthogonal Frequency-Division Multiplexing}
\newac{idft}{IDFT}{Inverse Discrete Fourier Transform}
\newac{lidar}{LiDAR}{Light Detection and Ranging}
\newac{imu}{IMU}{Inertial Measurement Unit}
\newac{speb}{SPEB}{square position error bound}
\newac[plural=EFIMs,firstplural=Fisher information matrices (EFIMs)]{efim}{EFIM}{Fisher information matrix}
\newac{ne}{NE}{Nash equilibrium}
\newac{mse}{MSE}{mean squared error}
\newac{toa}{TOA}{time-of-arrival}
\newac{snr}{SNR}{signal-to-noise ratio}
\newac{lan}{LAN}{local area network}
\newac{psd}{PSD}{positive semidefinite}
\newac{pd}{PD}{positive definite}
\newac{wrt}{w.r.t.}{with respect to}
\newac{lhs}{L.H.S.}{left hand side}
\newac{wp1}{w.p.1}{with probability 1}
\newac{kkt}{KKT}{Karush-Kuhn-Tucker}
\newac{wlog}{w.l.o.g.}{without loss of generality}
\newac{mle}{MLE}{maximum likelihood estimation}
\newac{gps}{GPS}{global positioning system}
\newac{rssi}{RSSI}{received signal strength indicator}
\newac{mimo}{MIMO}{multiple-input multiple-output}
\newac{csi}{CSI}{channel state information}
\newac{fdd}{FDD}{frequency division duplexing}
\newac{ms}{MS}{mobile station}
\newac{bs}{BS}{base station}
\newac{d2d}{D2D}{device-to-device}
\newac{slnr}{SLNR}{signal-to-interference-leakage-and-noise-ratio}
\newac{ula}{ULA}{uniform linear antenna array}
\newac{pas}{PAS}{power angular spectrum}
\newac{mmse}{MMSE}{minimum mean square error}
\newac{zf}{ZF}{zero-forcing}
\newac{rzf}{RZF}{regularized zero-forcing}
\newac{as}{AS}{angular spread}
\newac{aod}{AOD}{angle of departure}
\newac{iid}{i.i.d.}{independent and identically distributed} 
\newac{sinr}{SINR}{signal-to-interference-and-noise ratio}
\newac{tdd}{TDD}{time-division duplex}
\newac{rvq}{RVQ}{random vector quantization}
\newac{rhs}{R.H.S.}{right hand side}
\newac{mrc}{MRC}{maximum ratio combining}
\newac{cdf}{CDF}{cumulative distribution function}
\newac{a.s.}{a.s.}{almost surely}
\newac{los}{LOS}{line-of-sight}
\newac{jsdm}{JSDM}{joint spatial division and multiplexing}
\newac{map}{MAP}{maximum a posteriori}
\newac{klt}{KLT}{Karhunen-Lo\`eve Transform}
\newac{lbe}{LBE}{link bargaining equilibrium}
\newac{se}{SE}{Stackelberg equilibrium}
\newac{uav}{UAV}{unmanned aerial vehicle}
\newac{nlos}{NLOS}{non-line-of-sight}
\newac{pdf}{PDF}{probability density function}
\newac{em}{EM}{expectation-maximization}
\newac{knn}{KNN}{$k$-nearest neighbor}
\newac{svd}{SVD}{singular value decomposition}
\newac{nmf}{NMF}{non-negative matrix factorization}
\newac{umf}{UMF}{unimodality-constrained matrix factorization}
\newac{rmse}{RMSE}{rooted mean squared error}
\newac{olos}{OLOS}{obstructed line-of-sight}
\newac{mmw}{mmW}{millimeter wave}
\newac{ber}{BER}{bit error rate}
\newac{rss}{RSS}{received signal strength}
\newac{lp}{LP}{linear program}
\newac{ufw}{U-FW}{unimodal Frank-Wolfe}
\newac{utf}{UTF}{unimodality-constrained tensor factorization}
\newac{fw}{FW}{Frank-Wolfe}
\newac{iot}{IoT}{Internet-of-Things}
\newac{mae}{MAE}{mean absolute error}
\newac{crb}{CRB}{Cram\'er-Rao bound}
\newac{aoa}{AoA}{angle of arrival}
\newac{wcl}{WCL}{weighted centroid localization}

\usepackage{babel} 
\usepackage{array}

\usepackage{amsthm}
\theoremstyle{plain}
\newtheorem{mythm}{\theoremname}  

\theoremstyle{remark}

\theoremstyle{proposition}
\newtheorem{myprop}{Proposition}

\makeatother

\usepackage{babel}
\addto\captionsamerican{\renewcommand{\lemmaname}{Lemma}}
\addto\captionsamerican{\renewcommand{\remarkname}{Remark}}
\addto\captionsamerican{\renewcommand{\theoremname}{Theorem}}
\addto\captionsenglish{\renewcommand{\algorithmname}{Algorithm}}
\addto\captionsenglish{\renewcommand{\lemmaname}{Lemma}}
\addto\captionsenglish{\renewcommand{\remarkname}{Remark}}
\addto\captionsenglish{\renewcommand{\theoremname}{Theorem}}
\providecommand{\lemmaname}{Lemma}
\providecommand{\remarkname}{Remark}
\providecommand{\theoremname}{Theorem}

\begin{document}
\title{Blind Radio Mapping via Spatially Regularized Bayesian Trajectory
Inference}
\maketitle

\begin{abstract}
Radio maps enable intelligent wireless applications by capturing the
spatial distribution of channel characteristics. However, conventional
construction methods demand extensive location-labeled data, which
are costly and impractical in many real-world scenarios. This paper
presents a blind radio map construction framework that infers user
trajectories from indoor \ac{mimo}-\ac{ofdm} channel measurements
without relying on location labels. It first proves that \ac{csi}
under \ac{nlos} exhibits spatial continuity under a quasi-specular
environmental model, allowing the derivation of a \ac{csi}-distance
metric that is proportional to the corresponding physical distance.
For rectilinear trajectories in Poisson-distributed \ac{ap} deployments,
it is shown that the \ac{crlb} of localization error vanishes asymptotically,
even under poor angular resolution. Building on these theoretical
results, a spatially regularized Bayesian inference framework is developed
that jointly estimates channel features, distinguishes \ac{los}/\ac{nlos}
conditions and recovers user trajectories. Experiments on a ray-tracing
dataset demonstrate an average localization error of 0.68 m and a
beam map reconstruction error of 3.3\%, validating the effectiveness
of the proposed blind mapping method.
\end{abstract}


\glsresetall

\section{Introduction}

Radio maps link physical locations with channel characteristics, enabling
new methodologies for \ac{csi} acquisition, tracking, and prediction
\cite{ZenChe:J24,RomDan:J22,ZenXuX:J21}. {Conventional
methods for constructing radio maps predominantly rely on labeled
datasets, where the CSI measurements are associated with accurate
location information} \cite{TimSub:J23,ShrFu:J22,YapKan:J23}. {However,
obtaining such location-labeled CSI data is costly and labor-intensive,
often requiring dedicated positioning infrastructure, or manual calibration
\cite{TeqRom:J21,LeviRon:J21,XinChe:J24}.} These limitations become
particularly pronounced in scenarios demanding rapid deployment or
frequent updates, such as dynamic urban environments, shopping malls,
and exhibition halls that require frequent reconfiguration. Consequently,
developing radio map construction methods that do not rely on location
labels is crucial for enhancing the scalability, flexibility, and
practicality of wireless network intelligence.

There were some attempts on reducing the reliance on location-labeled
measurements. For instance, the work {\cite{SatSut:J21}}
employed Kriging-based space--frequency interpolation to construct
radio maps from a limited set of labeled data. Similarly, the work
{\cite{LiLiao:J23}} developed a geometry-driven
matrix completion approach, leveraging virtual anchor modeling and
spatial consistency to extrapolate multipath angles and delays. Other
methods have incorporated Bayesian inference; for example, the work
{\cite{WanZhu:J24}} investigated sparse
sampling combined with Bayesian learning to optimize sampling locations
and account for shadow fading in spectrum recovery. Note that these
methods are based on the methodologies of interpolation and supervised
learning, and thus, they still necessitate a sufficient amount of
location-labeled data, which requires labor cost in radio map construction
and update.

Another line of research aims to reduce the dependence on location-labeled
data by exploiting additional sensors or side information from the
environment. For example, the work {\cite{WanMao:J19}}
investigated \ac{uwb}-assisted positioning, where synchronized angle
and time measurements enabled accurate localization. The work {\cite{ShiChe:J24}}
employed \ac{lidar}-based scanning to achieve fine-grained environmental
perception and self-localization, while the work{
\cite{ChoJeo:J22}} utilized \ac{imu} measurements to estimate relative
positions from accelerometer data. However, these approaches require
specialized hardware deployments, incur additional costs, and are
often impractical in large-scale or resource-constrained scenarios.
Recently, channel charting has emerged as a promising paradigm for
radio map construction, wherein high-dimensional \ac{csi} is embedded
into low-dimensional manifolds using neural architectures such as
triplet-based networks {\cite{TanPal:C23}},
Siamese networks {\cite{StaYam:J23}}, and
bilateration-based networks{ \cite{TanPal:J25}}.
Despite their potential, these methods typically require a small number
of labeled points to align the latent and geometric spaces. Moreover,
such alignment may fail to faithfully reflect the underlying physical
geometry, especially in dense multipath \ac{nlos} environments.

This paper focuses on recovering the location along an unknown trajectory
that measures the {\ac{mimo}-\ac{ofdm}}
channels in an indoor environment where there could be \ac{nlos}
regions. With the recovered location labels, a radio map can be constructed
by associating the recovered locations with the channel measurements.
The two main challenges addressed in this paper are: (i) {\em How to infer the location information from a sequence of channel measurements},
where there are multi-paths and fading that give randomness to the
channel, and (ii) {\em How to improve the location inference in NLOS by exploiting the spatially correlation of the channel}.

To tackle these challenges, we develop a spatially regularized Bayesian
framework for trajectory inference from unlabeled channel measurements.
The following two principles are investigated. First, we establish
a Bayesian model to identify \ac{los} measurements, where the angle
information can be extracted more accurately compared to the \ac{nlos}
counterpart. Second, we develop a CSI distance metric for \ac{nlos}
measurements, where the metric is shown to be proportional to the
physical distance of the measurements, and hence, it leads to a spatial
regularization for Bayesian trajectory inference.

Specifically, the following technical contributions are made:
\begin{itemize}
\item We establish a {\em quasi-specular environment model} to understand
whether there exists a radio signature that is statistically and locally
{\em continuous} in the physical space given the randomness nature
of the wireless channel. We derive a spatial continuity theorem and
find a theoretical CSI distance metric that is proven to be proportional
to the physical distance scaled by the bandwidth if the measurements
are obtained in a fully scattered \ac{nlos} scenario.
\item We investigate whether it is theoretically possible to recover a rectilinear
trajectory using a sufficient number of measurements despite an arbitrarily
poor accuracy in \ac{aoa} estimation under \ac{nlos}. We theoretically
show that this is possible under certain conditions on \ac{ap} topology,
and the corresponding localization error decays as $\mathcal{O}(1/T)$
for $T$ measurements.
\item We formulate a spatially regularized Bayesian framework for trajectory
inference, in which the user trajectory is recovered jointly with
LOS/NLOS assignment and channel feature estimation.
\item We evaluate the proposed method on a ray-tracing dataset with \ac{ula}
configuration. The proposed method achieves an average localization
error of 0.68 m (1.07 m in NLOS), a LOS/NLOS identification error
of 2\%, and a relative error of 3.3\% in constructing the MIMO beam
map based on the estimated trajectory.
\end{itemize}

$\quad$The remainder of this paper is organized as follows. Section~\ref{sec:System-Model}
introduces the propagation model in \ac{mimo}-\ac{ofdm} systems,
the user mobility model, and the Bayesian framework for trajectory
inference. Section~\ref{sec:Spatial-Continuity} develops the quasi-specular
environment model, establishes the spatial continuity property of
the wireless channel in NLOS region, and derives the \ac{crlb} of
the localization error under limited and unlimited regions. Section~\ref{sec:Problem-Alg}
presents the formulation of the spatially regularized Bayesian problem
and the design of the trajectory inference algorithm. Section~\ref{sec:Experiments}
reports experimental evaluations, and Section~\ref{sec:Conclusion}
concludes the paper.

\section{System Model}

\label{sec:System-Model}

\begin{figure}[t]
\centering{}\includegraphics[width=1\columnwidth]{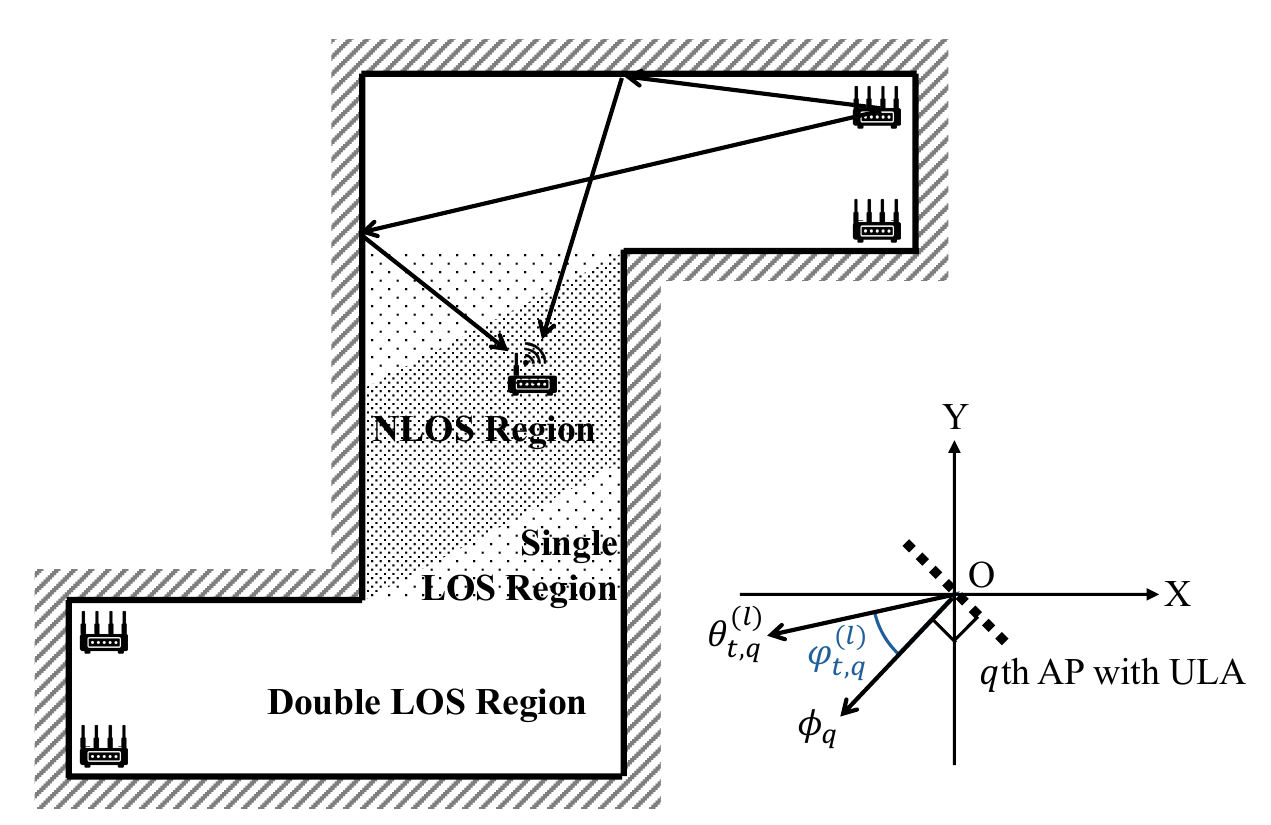}\caption{{An example indoor environment with NLOS regions.
\label{fig:antenna}}}
\end{figure}

\subsection{Propagation Model in MIMO-OFDM Networks}

Consider a single-antenna mobile user moving in an indoor environment.
There are $Q$ \acpl{ap} located at known positions $\mathbf{o}_{1},\mathbf{o}_{2},\dots,\mathbf{o}_{Q}\in\mathbb{R}^{2}$.
Each \ac{ap} is equipped with a \ac{ula}
consisting of $N_{\mathrm{t}}$ antennas. For
the ease of elaboration, we assume no angular ambiguity in the antenna
geometry. For example, assume that the \acpl{ap} are installed on
a wall or the antenna arrays are installed with a rear panel, such
that there is approximately no signal arriving from the back of the
\ac{ap}.

Consider the XOY coordinate system as shown in Figure~\ref{fig:antenna}.
For each \ac{ap} $q$, denote $\phi_{q}$ as the angle of the reference
direction that is orthogonal to the array axis of the antenna. For
each user position $\mathbf{x}_{t}$ at time $t$, denote $\theta_{t,q}^{(l)}$
as the \ac{aod} of the $l$th propagation path from the $q$th AP
at position $\mathbf{o}_{q}$ to the user at position $\mathbf{x}_{t}$
as shown in Figure~\ref{fig:antenna}.
As a result, the \emph{relative} \ac{aod} $\varphi_{t,q}^{(l)}\in(-\pi/2,\pi/2)$
\ac{wrt} the reference direction $\phi_{q}$ satisfies $\theta_{t,q}^{(l)}=(\varphi_{t,q}^{(l)}+\phi_{q})\mod2\pi$.

Assuming far-field propagation, the steering vector at the $q$th
\ac{ap} for a propagation path with a
relative \ac{aod} $\varphi$ is defined as
\begin{equation}
\mathbf{a}(\varphi)=\left[1,e^{-j\frac{2\pi}{\lambda}\Delta\sin\varphi},\ldots,e^{-j\frac{2\pi}{\lambda}(N_{\mathrm{t}}-1)\Delta\sin\varphi}\right]^{\mathrm{T}}\label{eq:steering}
\end{equation}
where $\Delta$ denotes the inter-element spacing, $\lambda=\frac{c}{f_{c}}$
is the wavelength at the carrier frequency $f_{c}$, and $c=3\times10^{8}\ \mathrm{m/s}$
is the speed of light.

Consider an \ac{ofdm} system with $M$
subcarriers ($M>N_{\mathrm{t}}$) and bandwidth $B$. At time slot
$t$, the baseband equivalent \ac{mimo}
channel $\mathbf{h}_{t,q}^{(m)}\in\mathbb{C}^{N_{\text{t}}}$ of the
$m$th subcarrier at the $q$th \ac{ap}
is given by\footnote{Here, we have implicitly assumed that the delays $\tau_{t,q}^{(\ell)}$
are discretized following a tap-delay model.}
\begin{equation}
\mathbf{h}_{t,q}^{(m)}=\sum_{\ell=1}^{L}\kappa_{t,q}^{(\ell)}e^{-j2\pi\frac{m}{M}B\tau_{t,q}^{(\ell)}}\mathbf{a}(\varphi_{t,q}^{(\ell)})\label{eq:MIMO-OFDM channel model}
\end{equation}
where $L$ is the number of paths from the $q$th \ac{ap}
to the user at position $\mathbf{x}_{t}$. For the path $l$ between
\ac{ap} $q$ and location $\mathbf{x}_{t}$,
$\kappa_{t,q}^{(\ell)}\in\mathbb{C}$ denotes the complex gain, and
$\tau_{t,q}^{(\ell)}\in\mathbb{R}_{+}$ represents the propagation
delay.

By stacking the vectors \eqref{eq:MIMO-OFDM channel model}, the \ac{ofdm}
channel measured at the $q$th \ac{ap} at time slot $t$ is represented
by
\[
\mathbf{H}_{t,q}=[\mathbf{h}_{t,q}^{(1)},\mathbf{h}_{t,q}^{(2)},\dots,\mathbf{h}_{t,q}^{(M)}]\in\mathbb{C}^{N_{\mathrm{t}}\times M}.
\]

\subsection{Mobility Model}

Consider that the slot duration is even. The mobility of the user
may follow a distribution where a big jump from $\mathbf{x}_{t-1}$
to a far away position $\mathbf{x}_{t}$ is rare. To capture such
a prior information, a widely adopted model is the Gauss-Markov model
\cite{He:J18} 
\begin{equation}
\mathbf{x}_{t}-\mathbf{x}_{t-1}=\gamma\left(\mathbf{x}_{t-1}-\mathbf{x}_{t-2}\right)+(1-\gamma)\delta\bar{\mathbf{v}}+\sqrt{1-\gamma^{2}}\,\delta\bm{\epsilon}\label{eq:mobility-model}
\end{equation}
which captures the temporal correlation in both the position $\mathbf{x}_{t}$
and the speed $\mathbf{x}_{t}-\mathbf{x}_{t-1}$, where $0<\gamma\leq1$
is the velocity correlation coefficient, $\delta$ is the time slot
duration, $\bar{\mathbf{v}}$ is the average velocity, and $\bm{\epsilon}\sim\mathcal{N}(\bm{0},\sigma_{\mathrm{v}}^{2}\mathbf{I})$
models the randomness.

We adopt a graph-based approach to convert \eqref{eq:mobility-model}
into a probability model. Specifically, the indoor environment is
represented as a graph $\mathcal{G}=(\mathcal{V},\mathcal{E})$, where
$\mathcal{V}$ denotes the set of possible discretized positions $\mathbf{p}_{i}\in\mathbb{R}^{2}$
placed uniformly over the area of the indoor environment. Denote $D_{\text{m}}$
as the maximum travel distance within a single time slot. The set
$\mathcal{E}$ consists of edges, such that there is an edge between
$\mathbf{p}_{i}$ and $\mathbf{p}_{j}$ if $\|\mathbf{p}_{j}-\mathbf{p}_{i}\|_{2}\leq D_{\text{m}}$.
Each node is also considered adjacent to itself, i.e., $(\mathbf{p}_{i},\mathbf{p}_{i})\in\mathcal{E}$.

Let $\mathcal{N}_{i}=\{\mathbf{p}_{j}:\|\mathbf{p}_{j}-\mathbf{p}_{i}\|_{2}\leq D_{\text{m}}\}$
denote the set of neighbors of node $\mathbf{v}_{i}$ (including itself).
The user trajectory is modeled as a sequence of positions constrained
to the graph $\mathcal{G}$.

To construct the transition probabilities on the discrete graph, we
evaluate the continuous-state transition density at each neighbor
and normalize across all feasible neighbors. Specifically, the probability
that the user moves from node $\mathbf{p}_{n}$ at time $t-2$ and
node $\mathbf{p}_{i}$ at time $t-1$ to a neighbor node $\mathbf{p}_{j}\in\mathcal{N}_{i}$
at the next time $t$ is given by 
\begin{equation}
\mathbb{P}(\mathbf{p}_{j}|\mathbf{p}_{i},\mathbf{p}_{n})=\frac{a_{nij}}{\sum_{k\in\mathcal{N}_{i}}a_{nik}}\label{eq:discretized-mobility-model}
\end{equation}
where the normalization $1/\sum_{k\in\mathcal{N}_{i}}a_{nik}$ guarantees
that for any given $\mathbf{p}_{i}$ and $\mathbf{p}_{n}$, $\sum_{j\in\mathcal{N}_{i}}\mathbb{P}(\mathbf{p}_{j}|\mathbf{p}_{i},\mathbf{p}_{n})=1$,
the probability $\mathbb{P}(\mathbf{p}_{j}|\mathbf{p}_{i},\mathbf{p}_{n})=0$
for $\mathbf{p}_{j}\notin\mathcal{N}_{i}$, and the factor
\begin{align*}
a_{nij} & =\frac{1}{2\pi(1-\gamma^{2})\delta^{2}\sigma_{\mathrm{v}}^{2}}\exp\Big(-\frac{1}{2(1-\gamma^{2})\delta^{2}\sigma_{\mathrm{v}}^{2}}\\
 & \qquad\times(\mathbf{p}_{j}+\gamma\mathbf{p}_{n}-(1+\gamma)\mathbf{p}_{i}-(1-\gamma)\delta\bar{\mathbf{v}})^{\mathrm{T}}\\
 & \qquad\times(\mathbf{p}_{j}+\gamma\mathbf{p}_{n}-(1+\gamma)\mathbf{p}_{i}-(1-\gamma)\delta\bar{\mathbf{v}})\Big)
\end{align*}
is given from the Gauss-Markov model \eqref{eq:mobility-model}. Based
on \eqref{eq:discretized-mobility-model}, one can evaluate the transition
probability $\mathbb{P}(\mathbf{x}_{t}|\mathbf{x}_{t-1},\mathbf{x}_{t-2})$.

\subsection{A Bayesian Framework for Trajectory Inference}

Denote $\mathcal{H}$ as a mapping from the \ac{mimo}-\ac{ofdm}
channel $\mathbf{H}_{t,q}$ to the observed radio signature $\mathbf{y}_{t,q}=\mathcal{H}(\mathbf{H}_{t,q})$
that will be exploited for trajectory inference. The development of
a specific mapping $\mathcal{H}$ will be discussed in Section \ref{sec:Problem-Alg}.
Assuming that the observations $\mathbf{y}_{t,q}$ are independent
across $t$ and $q$, a Bayesian model that describes the evolution
of the observation $\mathbf{y}_{t}=\{\mathbf{y}_{t,q}\}_{q}$ as a
function of the trajectory can be formulated as 
\begin{align}
p(\mathcal{Y}_{t},\mathcal{X}_{t}) & =p(\mathbf{y}_{t}|\mathbf{x}_{t})\mathbb{P}(\mathbf{x}_{t}|\mathbf{x}_{t-1},\mathbf{x}_{t-2})p(\mathcal{Y}_{t-1},\mathcal{X}_{t-1})\label{eq:p_Yt_Xt}\\
 & =\prod_{\tau=1}^{t}\prod_{q=1}^{Q}p(\mathbf{y}_{\tau,q}|\mathbf{x}_{\tau})\prod_{\tau=3}^{t}\mathbb{P}(\mathbf{x}_{\tau}|\mathbf{x}_{\tau-1},\mathbf{x}_{\tau-2}),\label{eq:p_Yt_Xt2}
\end{align}
where $\mathcal{X}_{t}=(\mathbf{x}_{1},\ldots,\mathbf{x}_{t})$ and
$\mathcal{Y}_{t}=(\mathbf{y}_{1},\ldots,\mathbf{y}_{t})$ are the
trajectory of the mobile user and the accumulated observations up
to time $t$, respectively and (\ref{eq:p_Yt_Xt2}) is obtained by
repeatedly applying the chain rule in (\ref{eq:p_Yt_Xt}). While the
mobility model $\mathbb{P}(\mathbf{x}_{t}|\mathbf{x}_{t-1},\mathbf{x}_{t-2})$
is given in (\ref{eq:discretized-mobility-model}), the main challenge
is to develop the conditional probability $p(\mathbf{y}_{t}|\mathbf{x}_{t})$
that will be discussed in Section \ref{sec:Problem-Alg}.

In the rest of the paper, we propose to extract the trajectory $\mathcal{X}_{t}$
by maximizing the parameterized likelihood (\ref{eq:p_Yt_Xt2}) with
a proper design of the radio signature mapping $\mathcal{H}$ that
extracts spatial information from the \ac{mimo}-\ac{ofdm} channels
and a cost function that exploits the property of the propagation.

\section{Spatial Continuity and Trajectory Identifiability}

\label{sec:Spatial-Continuity}

The key challenge of trajectory inference without any location labels
in an indoor environment is that the propagation is probably \ac{nlos}
where the path may arrive at any angle, and in this case, the \ac{aod}
may provide very little information on direction of the user. Moreover,
the path amplitude also barely contains any information on the propagation
distance due to the fading in indoor with rich scattering.

In this section, we try to understand some intuitive theoretical principle
of blind trajectory inference by studying the properties of the channel
and the trajectory estimation in simplified and special scenarios.
First, for \ac{nlos}, we try to establish the spatial continuity
of the channel, where under the same propagation environment as to
be specified later, there exists a distance metric $D$ such that
the distance between the \ac{csi} is consistent with the physical
distance $d$ between the corresponding physical locations, \emph{i.e.},
$D(\mathbf{H}_{1},\mathbf{H}_{2})\propto d(\mathbf{x}_{1},\mathbf{x}_{2})$.
Second, we try to establish a simple scenario with theoretical guarantee
where user trajectories are identifiable, \emph{i.e.}, to perfectly
recover a simple trajectory under infinite amount of independent measurements.

\subsection{Quasi-Specular Environment Model}

\label{subsec:Quasi-specular-Environment-Model}

We first specify the environment model for mathematical tractability.
Consider the 2D indoor environment to be surrounded by a finite number
of \emph{quasi-specular surfaces} and a finite number of \emph{diffractive
scatters}. A \emph{quasi-specular surface} is defined as a patch that
roughly holds a specular property, where the patch absorbs some energy
of an incident wave, and reflects the wave with a majority energy
towards the direction with the emergence angle equal to the incident
angle and with a minor energy towards the directions dispersed around
the major reflected path as shown in Figure \ref{fig:SpatialContinuity}(a).
While this is also known as scattering in some literature \cite{OesCla:J03,AstDav:J02},
our model limits the range of scattering around the major reflected
path, whose emergence angle equal to the incident angle. As a result,
when a receiver (RX) locates at a position on the major reflected
path of a transmitter (TX), it may also receive a number of scattered
paths arriving from a similar angle as shown in Figure \ref{fig:SpatialContinuity}(b).
Based on the geometry, we can model the reflection with scattering
due to this patch using a mirror TX located at the symmetric position
of the true TX about the patch and a cluster of virtual TXs surrounding
the mirror TX, as shown in Figure \ref{fig:SpatialContinuity}(c).
Note that the scattering paths may arrive constructively or destructively
at the RX, modeling the fading phenomenon. Therefore, the amplitude
and phase of the received signal at the RX cannot be computed in a
deterministic way, but the major propagation delay and angle can still
be computed geometrically based on our model.

A \emph{diffractive scatter} models the diffraction phenomenon, where
a radio propagation path may bend when passing over the edge of an
obstacle as shown in Figure \ref{fig:SpatialContinuity}(d). Geometrically,
one can place a virtual TX with the same distance away from the diffractive
scatter as that of the true TX such that the RX, diffractive scatter,
and the virtual TX are on the same line as shown in Figure \ref{fig:SpatialContinuity}(e).
In our model, we assume diffraction exists only for a limited angle
range, meaning a diffractive path can only bend for a certain angle.
As a result, given a TX, each diffractive scatter corresponds to a
series of TXs located on an arc segment.

We only consider a finite number of reflections and diffractions.
Note that each mirror TX and the associated cluster of virtual TXs
can be reflected and diffracted again, resulting in double reflections,
reflection-then-diffraction, triple reflections, and so on. Note that,
the more reflections and diffractions, the further away a mirror TX
from the original TX is created due to the geometry relation as shown
in Figure \ref{fig:SpatialContinuity}(f), forming a \emph{lattice}
of mirror or virtual TX clusters. Thus, it is a natural assumption
to consider that the clusters of virtual TXs do not overlap in a simple
environment that has not so many patches and diffractive scatters.
A typical example that likely matches with the quasi-specular environment
model is a polygonal indoor office with several solid walls.

\begin{figure}[t]
\centering{}\includegraphics[width=1\columnwidth]{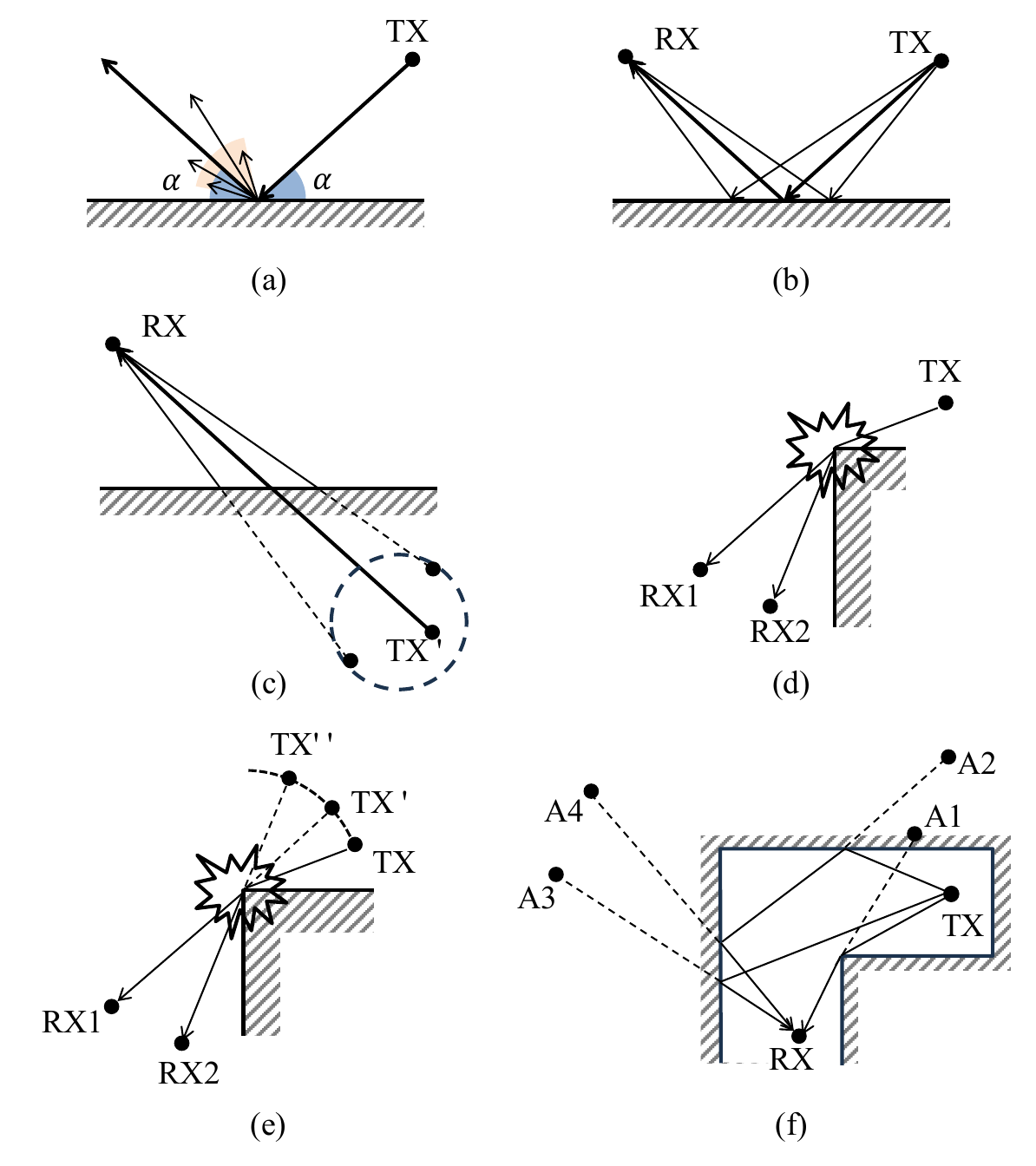}\caption{Illustration of quasi-specular environment model{\@.}
(a) The scattered paths spread over a limited range of angles. (b)
The RX receives a number of scattered paths arriving from a similar
angle. (c) It is equivalent to view that the RX receives signals from
a cluster of virtual TXs surrounding the mirror TX. (d) Diffraction.
(e) A series of virtual TXs on an arc segment. (f) A lattice of mirror
TXs A1, A2, A3, and A4.\label{fig:SpatialContinuity}}
\end{figure}

\subsection{Spatial Continuity}

The quasi-specular environment model inspires the construction of
radio signature $S_{i}$ for identifying position $\mathbf{p}_{i}$.
Specifically, consider to place a TX at position $\mathbf{p}_{i}$.
One can construct a lattice of mirror images of $\mathbf{p}_{i}$.
Note that if one moves the TX, its images also move. Thus, in general,
the lattice of mirror images differs from $\mathbf{p}_{i}$ to $\mathbf{p}_{j}$
unless there is perfect symmetry of the two distinct positions $\mathbf{p}_{i}$
and $\mathbf{p}_{j}$ in the environment, which is rare. As a result,
one can selected the $L$ mirror images that are the nearest from
the \ac{ap} as the spatial signature. More precisely, a \emph{conceptual}
radio signature $\mathcal{S}_{i}$ for position $\mathbf{p}_{i}$
can be constructed as $\mathcal{S}_{i}=\{(\tau_{i}^{(l)}-\tau_{i}^{(1)},\theta_{i}^{(l)}),l=1,2,\dots,L\}$,
where $\tau_{i}^{(l)}-\tau_{i}^{(1)}$ is the relative delay of the
$l$th shortest path to the delay of the first arrival path and $\theta_{i}^{(l)}$
is the \ac{aoa} of the $l$th shortest path at the RX. The reason
that absolute delays are excluded in $\mathcal{S}_{i}$ because time
synchronization is not assumed between TX and RX. In addition, the
amplitude information is also excluded in $\mathcal{S}_{i}$ because
the amplitude can be random due to various effects in practice, including
fading, antenna pattern, and body shadowing.

The essential problem to investigate here is \emph{whether the radio
signature $S$ is locally continuous in the physical space}. More
precisely, can we find a distance metric on $S$, where the distance
between $S_{i}$ and $S_{j}$ is consistent with the physical distance
between the two corresponding positions $\mathbf{p}_{i}$ and $\mathbf{p}_{j}$,
at least when the two positions are nearby? This problem is essential
for radio map construction, because it asserts whether one can abstract
the radio signature $S$ in a continuous physical area to a \emph{finite}
number of \emph{representations} $\{S_{i}\}$, such that for any continuous
position $\mathbf{p}$, there is a discrete position $\mathbf{p}_{i}$
nearby whose radio signature $S_{i}$ is close to the signature at
position $\mathbf{p}$. In addition, if spatial continuity exists,
it helps infer the latent physical topology from the \ac{csi} and
enables a blind construction of radio map from \ac{csi} measurements
without location labels.

We first specify the scenario where two nearby positions share the
same propagation environment. Denote the \ac{ap} position as point
$A$. Following the quasi-specular environment model, a lattice of
mirror positions $\mathcal{A}=\{A_{1},A_{2},\dots,A_{\bar{L}}\}$
can be formed based on sequential specular operations, where we only
consider a finite number of reflections and diffractions. Recall that,
for each mirror position $A_{l}\in\mathcal{A}$, it represents a cluster
of virtual mirror APs due to possible scattering and diffraction as
discussed in Section \ref{subsec:Quasi-specular-Environment-Model}.
For a position $\mathbf{p}_{i}$, only a subset $\mathcal{A}^{i}\subset\mathcal{A}$
of mirror images are visible from $\mathbf{p}_{i}$ because some paths
may be blocked as shown in Figure \ref{fig:SpatialContinuity}. Two
positions $\mathbf{p}_{i}$ and $\mathbf{p}_{j}$ are said to share
the same propagation environment if the corresponding visible subsets
are identical $\mathcal{A}^{i}=\mathcal{A}^{j}$.

We then formulate the problem starting from a single antenna case
for the ease of elaboration. Consider two nearby positions $\mathbf{p}_{1}$
and $\mathbf{p}_{2}$ that share the same propagation environment,
\emph{i.e.}, $\mathcal{A}^{1}=\mathcal{A}^{2}$. Following the model
in (\ref{eq:MIMO-OFDM channel model}), the $m$th subcarrier of the
\ac{ofdm} channel $\mathbf{h}_{1}$ at position $\mathbf{p}_{1}$
is given by 
\begin{equation}
h_{1}^{(m)}=\sum_{l=1}^{L}\kappa_{1}^{(l)}e^{-jm\frac{2\pi}{M}B\tau_{1}^{(l)}}\label{eq:channel-model-single-antenna}
\end{equation}
where, following our quasi-specular environment model, $\kappa_{1}^{(l)}$
and $\tau_{1}^{(l)}$ are, respectively and conceptually, the complex
path gain and delay for the path that associated with the $l$th mirror
\ac{ap} at $A_{l}\in\mathcal{A}^{1}$ in the visible subset. Likewise,
we can express the channel $h_{1}^{(m)}$ for position $\mathbf{p}_{2}$
that share the same subset of visible mirror images $\mathcal{A}^{2}=\mathcal{A}^{1}$.

Define 
\begin{equation}
R(u)=\frac{1}{M}\sum_{m=0}^{M-1}\mathbb{E}\{h_{1}^{(m)}h_{2}^{(m)*}\}e^{jm\frac{2\pi}{M}u}\label{eq:Ru}
\end{equation}
where the expectation is taken over the randomness of the complex
path gain $\kappa_{1}^{(l)}$, $\kappa_{2}^{(l)}$, and direction
$\theta'$ from $\mathbf{p}_{1}$ to $\mathbf{p}_{2}$. It follows
that $R(u)$ is the correlation function in the delay domain, because
$h_{i}^{(m)}$ represent the frequency domain channel, the product
$h_{1}^{(m)}h_{2}^{(m)*}$ computes the correlation, and the sum with
the exponential term is just the \ac{idft} formula, \emph{i.e.},
transforming the correlation to the time domain.

Consider the following theoretical metric 
\begin{equation}
\hat{u}(\mathbf{h}_{1},\mathbf{h}_{2})=\underset{u}{\arg\max}\;|R(u)|.\label{eq:hat-u}
\end{equation}
Denote $d=\|\mathbf{p}_{1}-\mathbf{p}_{2}\|_{2}$ as the physical
distance between the two positions, which have been assumed to share
the same propagation environment. Suppose that the complex path gains
$\kappa_{i}^{(l)}$ are zero-mean and uncorrelated among $l$, but
there is correlation among $i=1,2$, \emph{i.e.}, $\mathbb{E}\{\kappa_{1}^{(l)}\kappa_{2}^{(l')*}\}=0$
for $l\neq l'$, and $\mathbb{E}\{\kappa_{1}^{(l)}\kappa_{2}^{(l)*}\}=C(d)$
for all $l$. Typically, one may expect that the correlation $C(d)$
decreases in distance $d$. Then, we have the following result.

\begin{mythm}\label{thm:Spatial-Continuity}(Spatial Continuity)
Under the multipath model \eqref{eq:channel-model-single-antenna}
and a large $M$ approximation, we have $\hat{u}(\mathbf{h}_{1},\mathbf{h}_{2})\approx\frac{B}{c}d.$\end{mythm}

\begin{proof}See Appendix \ref{sec:Proof-of-Spatial-conti}.\end{proof}

Theorem \ref{thm:Spatial-Continuity} confirms that within the same
propagation environment, there exists a metric where the CSI distance
is consistent with the physical distance of the corresponding positions.
In a special case of using metric \eqref{eq:hat-u}, the two distances
are related by a factor of $B/c$.

The implication of Theorem \ref{thm:Spatial-Continuity} is two-fold:
First, one can approximate the CSI sampled randomly in a neighborhood
using the CSI sampled from a deterministic representative location
in a set $\mathcal{V}$. While this result may sound intuitive, Theorem
\ref{thm:Spatial-Continuity} finds a metric to quantify the CSI approximation
error in terms of distance $d$ away from the representative location.
Second, Theorem \ref{thm:Spatial-Continuity} provides a perspective
to convert the trajectory inference problem to a multidimensional
scaling (MDS) problem \cite{KonQi:J19} that embeds high-dimensional
CSI observation $\mathbf{y}_{t}$ onto a 2D space while preserving
the pairwise distances under a certain scaling. Such a philosophy
will be exploited for algorithm development in Section \ref{sec:Problem-Alg}.
\begin{rem}
[A Practical CSI Distance Metric]\label{rem:Practical-Distance-metric}
The metric $R(u)$ defined in (\ref{eq:Ru}) is a theoretical one
because it needs to compute the expectation. A practical metric can
be defined as 
\begin{equation}
\hat{R}(u)=\frac{1}{M}\sum_{m=0}^{M-1}h_{1}^{(m)}h_{2}^{(m)*}e^{jm\frac{2\pi}{M}u}.\label{eq:Ru-hat}
\end{equation}
\end{rem}
The two metrics (\ref{eq:Ru}) and (\ref{eq:Ru-hat}) are related
under a rich scattering environment where there are a large number
of independent paths. Then, the law of large numbers comes into play
to connect the expectation with the sum of paths (See (28) in Appendix
A). Therefore, Theorem \ref{thm:Spatial-Continuity} is only meaningful
in NLOS case.
\begin{rem}
[Multiple Antenna Extension]\label{rem:Multiple-Antenna-Extension}
A rigorous metric that exploits the angular domain information for
\ac{mimo}-\ac{ofdm} channel with analytical justification similar
to \eqref{eq:hat-u} remains unknown. The challenge is that the angular
discrepancy is inversely proportional to the total propagation distance,
which is relatively large compared to $d$. By contrast, the frequency
domain discrepancy that leads to \eqref{eq:hat-u} does not depend
on the absolute propagation distance. A naive extension to the multiple
antenna case without exploiting the angular domain discrepancy is
to apply \eqref{eq:Ru} to each antenna and yield $R_{n}(u)$ for
the $n$th antenna. Then, defining $R(u)=\frac{1}{N_{\text{t}}}\sum_{n=1}^{N_{\text{t}}}R_{n}(u)$,
and the same metric $\hat{u}(\mathbf{H}_{1},\mathbf{H}_{2})$ in \eqref{eq:hat-u}
follows and the same spatial continuity result applies.
\end{rem}

\subsection{Trajectory Identifiability}

\label{subsec:trajectory-idenfiability}

We now examine the problem of trajectory inference from another perspective.
Suppose that {the }\ac{aod}{
$\theta_{t,q}$ of the dominant path from }\ac{ap}{
$q$ to the user at location $\mathbf{x}_{t}$ follows a Gaussian
distribution
\begin{equation}
\theta_{t,q}\sim\mathcal{N}\left(\phi(\mathbf{x}_{t},\mathbf{o}_{q}),\sigma_{\theta}^{2}\right)\label{eq:theta-model}
\end{equation}
where $\phi(\mathbf{x}_{t},\mathbf{o}_{q})$ denotes the geometric
azimuth angle between the user location $\mathbf{x}_{t}$ and the
}\ac{ap}{ position $\mathbf{o}_{q}$. In
an LOS scenario, $\sigma_{\theta}^{2}$ tends to be small, whereas,
for NLOS, $\sigma_{\theta}^{2}$ can be large. While this model might
be a bit artificial for an indoor case, the problem we investigate
is highly non-trivial: \emph{for an arbitrarily large angular variance
$\sigma_{\theta}^{2}$, is it theoretically possible to recover a
rectilinear trajectory using a sufficient number of measurements?}}

The prior work \cite{ZheChe:J25} attempted
a similar problem in an outdoor case exploiting the power law that
relates the distance information with the signal strength. However,
in the indoor case, we do \emph{not} assume a power law of the signal
strength due to the more complicated propagation. Instead, we only
rely on the noisy and unbiased \ac{aod} information from \eqref{eq:theta-model}
to recover the trajectory, which leads to non-straight-forward answers.

Specifically, consider the Gauss-Markov mobility model in (\ref{eq:mobility-model})
in the continuous space with $\gamma=1$, which degenerates to a constant
speed rectilinear mobility
\begin{equation}
\mathbf{x}_{t}=\mathbf{x}+t\mathbf{v}\label{eq:constant-speed-mobility}
\end{equation}
where we only need to estimate the start position and the speed $(\mathbf{x},\mathbf{v})\in\mathbb{R}^{4}$
in a continuous space. Consequently, the distance $d(\mathbf{x}_{t},\mathbf{o}_{q})$
between the mobile location $\mathbf{x}_{t}$ and the $q$th \ac{ap}
location $\mathbf{o}_{q}$ is simplified as
$d_{t,q}(\mathbf{x},\mathbf{v})\triangleq\|\bm{l}_{q}(\mathbf{x})+t\mathbf{v}\|_{2}$,
where $\bm{l}_{q}(\mathbf{x})=\mathbf{x}-\mathbf{o}_{q}$ is the
direction from the $q$th \ac{ap} to
the initial position $\mathbf{x}$ of the trajectory.

Thus, the log-likelihood function $\log p(\mathcal{Y}_{T},\mathcal{X}_{T})$
in \eqref{eq:p_Yt_Xt2} simplifies to
\begin{align}
f(\bm{\psi}) & =\sum_{t=1}^{T}\sum_{q=1}^{Q}\bigg[-\ln(2\pi\sigma_{\theta}^{2})\label{eq:log-likelihood-constant-speed}\\
 & \qquad-\frac{1}{2\sigma_{\theta}^{2}}\left(\theta_{t,q}-\phi(\mathbf{x}+t\mathbf{v},\mathbf{o}_{q})\right)^{2}\Big]\nonumber 
\end{align}
{where $\bm{\psi}=(\mathbf{x},\mathbf{v})$
represents the mobility parameters, and the term $\log\mathbb{P}(\mathbf{x}_{t}|\mathbf{x}_{t-1},\mathbf{x}_{t-2})$
vanishes under the constant-speed mobility model.}

The focus here is to understand the fundamental limit of estimating
the mobility parameter $\bm{\psi}=(\mathbf{x},\mathbf{v})$ using
only \ac{aod} measurements.

\subsubsection{The Fisher Information Matrix}

The \ac{fim} $\mathbf{F}_{T,\psi}$ of $\bm{\psi}=(\mathbf{x},\mathbf{v})\in\mathbb{R}^{4}$
from the measurements over a duration $T$ can be computed as
\begin{align*}
\mathbf{F}_{T,\psi} & \triangleq\mathbb{E}\{-\nabla_{\bm{\psi}\bm{\psi}}^{2}f(\bm{\psi})\}\\
 & =\sum_{t,q}\frac{1}{\sigma_{\theta}^{2}}\nabla_{\bm{\psi}}\phi(\mathbf{x}+t\mathbf{v},\mathbf{o}_{q})(\nabla_{\bm{\psi}}\phi(\mathbf{x}+t\mathbf{v},\mathbf{o}_{q}))^{\mathrm{T}}.
\end{align*}
The derivative $\nabla_{\bm{\psi}}\phi(\mathbf{x}+t\mathbf{v},\mathbf{o}_{q})$
is derived as
\begin{align*}
\nabla_{\bm{\psi}}\phi(\mathbf{x}+t\mathbf{v},\mathbf{o}_{q}) & =\frac{1}{d_{t,q}^{2}(\mathbf{x},\mathbf{v})}\left[\begin{array}{c}
1\\
t
\end{array}\right]\\
 & \qquad\otimes\left(\begin{bmatrix}0 & -1\\
1 & 0
\end{bmatrix}(\bm{l}_{q}(\mathbf{x})+t\mathbf{v})\right).
\end{align*}

Thus, the \ac{fim} can be expressed as
\begin{align}
\mathbf{F}_{T,\psi} & =\sum_{t,q}\frac{1}{\sigma_{\theta}^{2}d_{t,q}^{4}(\mathbf{x},\mathbf{v})}\begin{bmatrix}1 & t\\
t & t^{2}
\end{bmatrix}\label{eq:F-Txtil}\\
 & \otimes\left(\lVert\bm{l}_{q}(\mathbf{x})+t\mathbf{v}\rVert^{2}\mathbf{I}-(\bm{l}_{q}(\mathbf{x})+t\mathbf{v})(\bm{l}_{q}(\mathbf{x})+t\mathbf{v})^{\text{T}}\right)\nonumber 
\end{align}
in which, $\otimes$ is the Kronecker product.

For an unbiased estimator $\hat{\bm{\psi}}$, the \ac{mse} is lower
bounded by $\mathbb{E}\{\|\hat{\bm{\psi}}-\bm{\psi}\|^{2}\}\geq\mathrm{tr}\{\mathbf{F}_{T,\psi}^{-1}\}$,
where $\mathrm{tr}\{\mathbf{F}_{T,\psi}^{-1}\}$ is the \ac{crlb}
of estimating $\bm{\psi}=(\mathbf{x},\mathbf{v})$. Similarly, we
define the \acpl{fim} $\mathbf{F}_{T,x}=\mathbb{E}\left\{ -\nabla_{\mathbf{x}\mathbf{x}}^{2}f(\bm{\psi})\right\} $
and $\mathbf{F}_{T,v}=\mathbb{E}\left\{ -\nabla_{\mathbf{v}\mathbf{v}}^{2}f(\bm{\psi})\right\} $,
which are the diagonal blocks of $\mathbf{F}_{T,\psi}$ and are associated
with the \ac{crlb} $B(\mathbf{x})=\mathrm{tr}\{\mathbf{F}_{T,x}^{-1}\}$
and \ac{crlb} $B(\mathbf{v})=\mathrm{tr}\{\mathbf{F}_{T,v}^{-1}\}$
for the parameters $\mathbf{x}$ and $\mathbf{v}$, respectively.

\subsubsection{\ac{ap} Deployed in a Limited Region}

We investigate the case where the \acpl{ap} are deployed in a limited
region, but the measurement trajectory is allowed to go unbounded
as $T$ goes to infinity. Signals can always be collected by the \acpl{ap}
regardless of the distance. As a result, an infinite amount of measurements
can be collected as $T\to\infty$.

It is observed that $\mathbf{F}_{T,\psi}\prec\mathbf{F}_{T+1,\psi}$,
indicating that the Fisher information is strictly increasing from
Lemma 4 in Appendix B, provided that $\bm{l}_{q}(\mathbf{x})$ and
$\mathbf{v}$ are linear independent for at least one $q$.

However, it is somewhat surprising that the \ac{crlb} for $\mathbf{x}_{}$
and $\mathbf{v}$ does not decrease to zero as $T\to\infty$, despite
the infinitely increasing amount of independent data.

Specifically, assume that the trajectory $\mathbf{x}_{t}$ does not
pass any of the \ac{ap} location $\mathbf{o}_{q}$, and hence, $d_{\min}=\min_{t,q}\{d_{t,q}(\mathbf{x},\mathbf{v})\}>0$
for all $t,q$.

\begin{myprop}\label{prop:LB-F-x}The \ac{crlb} of $\mathbf{x}$
satisfies $B(\mathbf{x})=\mathrm{tr}\{\mathbf{F}_{T,x}^{-1}\}\geq\bar{\Delta}_{T,x}$
with equality achieved when $d_{t,q}=d_{\min}$ and $\phi_{t,q}=0$
for all $t,q$. In addition, $\bar{\Delta}_{T,x}$ is strictly decreasing
in $T$, provided that at least two vectors in $\{\bm{l}_{1},\bm{l}_{2},\dots,\bm{l}_{Q},\mathbf{v}\}$
are linear independent, but $\bar{\Delta}_{T,x}$ converges to a strictly
positive number as $T\to\infty$.\end{myprop}

\begin{proof}See Appendix \ref{sec:Proof-of-limited-x}.\end{proof}

Proposition \ref{prop:LB-F-x} suggests that the \ac{crlb} of $\mathbf{x}$
cannot decrease to zero even when we estimate only two parameters
for the initial location $\mathbf{x}\in\mathbb{R}^{2}$ based on {\em infinite}
\ac{aod} measurements collected over an infinite geographical horizon
as $T\to\infty$.

Through the development of the proof, a physical interpretation of
Proposition \ref{prop:LB-F-x} can be given as follows. As $T$ increases,
the distances $d_{t,q}(\mathbf{x},\mathbf{v})=\|\mathbf{x}_{t}-\mathbf{o}_{q}\|_{2}$
grow larger because the user moves away from the \acpl{ap}. For a
position $\mathbf{x}_{t}$ at a sufficiently large distance, the term
$\mathbf{x}_{t}-\mathbf{o}_{q}$ approximates $\mathbf{x}_{t}$ since
$\|\mathbf{x}_{t}\|\gg\|\mathbf{o}_{q}\|$. Consequently{,
the angle measurement $\phi(\mathbf{x}_{t},\mathbf{o}_{q})$ changes
very little as $t\to\infty$, making successive }\ac{aod}{
observations almost indistinguishable regardless of the }\ac{ap}{
locations $\mathbf{o}_{q}$. Although the Fisher information matrix
$\mathbf{F}_{T,x}$ increases with $T$, the increment $\mathbf{F}_{T+1,x}-\mathbf{F}_{T,x}$
decays rapidly. Hence, while the CRLB decreases monotonically, it
converges to a strictly positive lower bound.}

We obtain a similar conclusion for estimating the velocity variable
$\mathbf{v}$.

\begin{myprop}\label{prop:LB-F-v}The \ac{crlb} of $\mathbf{v}$
satisfies $B(\mathbf{v})=\mathrm{tr}\{\mathbf{F}_{T,v}^{-1}\}\geq\bar{\Delta}_{T,v}$,
with equality achieved when $d_{t,q}=d_{\min}$ and $\phi_{t,q}=0$
for all $t,q$. In addition, 
\begin{align*}
\bar{\Delta}_{T,v} & \rightarrow C_{v}=\frac{d_{\mathrm{min}}^{2}\sigma_{\mathrm{n}}^{2}}{G_{1}N_{\mathrm{t}}(N_{\mathrm{t}}^{2}-1)}\Big(\sum_{q=1}^{Q}s_{\infty,q}^{(2)}\|\mathbf{P}_{v}^{\bot}\bm{l}_{q}(\mathbf{x})\|^{2}\Big)^{-1}
\end{align*}
 as $T\rightarrow\infty$, where $G_{1}$ depends on the antenna configuration,
$\sigma_{\mathrm{n}}^{2}$ is the signal noise, and
\[
s_{\infty,q}^{(2)}=\lim_{T\to\infty}\sum_{t=1}^{T}\frac{t^{2}}{d_{t,q}^{4}(\mathbf{x},\mathbf{v})},\:\mathbf{P}_{v}^{\bot}=\mathbf{I}-\mathbf{v}\mathbf{v}^{T}/\|\mathbf{v}\|^{2}
\]
in which, the parameter $s_{\infty,q}^{(2)}$ is upper bounded by
$1/\rho^{4}\lim_{T\to\infty}\sum_{t=1}^{T}1/t^{2}\approx\frac{\pi^{2}}{6\rho^{4}}$,
where $\rho>0$ is sufficiently small such that $d_{t,q}(\mathbf{x},\mathbf{v})>\rho t$
for all $t\geq1$.\end{myprop}

\begin{proof}See Appendix \ref{sec:Proof-of-limited-v}.\end{proof}

Propositions \ref{prop:LB-F-x} and \ref{prop:LB-F-v} suggest that,
under a finite number of \acpl{ap} in a limited region, one cannot
perfectly recover a trajectory even for a simple constant speed rectilinear
mobility under infinite measurements.

Proposition \ref{prop:LB-F-v} quantifies the fundamental limit to
the estimation accuracy that is affected by the spatial distribution
of the \acpl{ap} and the nature of the \ac{aod} measurements. Specifically,
{the term $\|\mathbf{P}_{v}^{\perp}\,\boldsymbol{l}_{q}(\mathbf{x})\|$
is the length of the $q$th }\ac{ap}{ position
vector component orthogonal to the velocity vector $\mathbf{v}$.
If all }\acpl{ap}{ lie nearly collinear
with $\mathbf{v}$, then $\mathbf{P}_{v}^{\perp}\,\boldsymbol{l}_{q}\approx\mathbf{0}$
and the sum $\sum_{q}\|\mathbf{P}_{v}^{\perp}\,\boldsymbol{l}_{q}\|^{2}$
is small, providing little information in directions orthogonal to
motion. By contrast, a wider angular spread of }\acpl{ap}{
(}\acpl{ap}{ surrounding the agent or spanning
different angles relative to $\mathbf{v}$) increases $\sum_{q}\|\mathbf{P}_{v}^{\perp}\,\boldsymbol{l}_{q}\|^{2}$,
improving geometric conditioning and reducing the CRLB. The equality
case $\phi_{t,q}=0$ corresponds to each }\ac{ap}{
being \textit{broadside} to $\mathbf{v}$, which maximizes information
gain. Limited region deployments exhibit geometric dilution of precision,
causing a nonzero lower bound on estimation error. Moreover, the lower
bound $C_{v}$ decreases with higher effective }\ac{snr}{
($G_{1}/\sigma_{\mathrm{n}}^{2}$) and scales as $\mathcal{O}(1/N_{t}^{3})$
with increasing $N_{t}$, as larger ULA apertures improve angular
resolution and array gain---practically reducing the }\ac{crlb}{
by a factor of eight when doubling the number of transmit antennas.}

\subsubsection{\ac{ap} Deployed in an Unlimited Region}

We now consider a theoretical scenario in which the \acpl{ap}
are distributed over an unbounded region according to a homogeneous
\ac{ppp} with density $\kappa$. Despite
the infinite spatial domain, the user is restricted to connecting
only with those \acpl{ap} located within
a fixed connectivity radius $R$. Consequently, the number of connected
\acpl{ap} at any time remains finite,
with the average number per time slot given by $\bar{Q}=\kappa\pi R^{2}$.
We are interested in the asymptotic behaviour of the \acpl{crlb},
as $T\to\infty$. This is to understand the error decrease rate, i.e.,
how fast the error may decrease as we increase the number of observations
$T$.

Remarkably, even though the set of active \acpl{ap}
at any instant is always limited, the estimation lower bound for user
state parameters vanishes in the limit of long observation duration.

\begin{mythm}\label{thm:LB-F-xv-un}Assume that the minimum distance
to the nearest \ac{ap} is greater than $r_{0}$ along the trajectory.\footnote{In practice, the parameter $r_0$ can be understood as the height of the  antenna. More rigorously, we should employ a 3D model to compute the distance $d_{t,q}$, but the asymptotic result would be the same.}
The \ac{crlb} of $\mathbf{x}$ satisfies $B(\mathbf{x})=\mathrm{tr}\{\mathbf{F}_{T,x}^{-1}\}\leq\tilde{\Delta}_{T,x}$
and as $T\to\infty$
\[
T\tilde{\Delta}_{T,x}\rightarrow\frac{16\sigma_{\mathrm{n}}^{2}}{\kappa\pi(r_{0}^{-2}-R^{-2})G_{1}N_{\mathrm{t}}(N_{\mathrm{t}}^{2}-1)}.
\]
The \ac{crlb} of $\mathbf{v}$ satisfies $B(\mathbf{v})=\mathrm{tr}\{\mathbf{F}_{T,v}^{-1}\}\leq\tilde{\Delta}_{T,v}$
and as $T\to\infty$
\[
T(T+1)(2T+1)\tilde{\Delta}_{T,v}\rightarrow\frac{96\sigma_{\mathrm{n}}^{2}}{\kappa\pi(r_{0}^{-2}-R^{-2})G_{1}N_{\mathrm{t}}(N_{\mathrm{t}}^{2}-1)}.
\]
\end{mythm}

\begin{proof}See Appendix \ref{sec:Proof-of-unlimited-x}.\end{proof}

From the above theorem, it is evident that the \ac{crlb}
for the initial position $\mathbf{x}$ decays as $\mathcal{O}(1/T)$,
while the \ac{crlb} for the velocity
$\mathbf{v}$ decreases at the faster rate of $\mathcal{O}(1/T^{3})$.
This demonstrates that velocity estimation becomes asymptotically
much more accurate than position estimation as the observation window
grows. Furthermore, both a larger connectivity radius $R$ and a higher
\ac{ap} density $\kappa$ within the
coverage region enhance the achievable estimation accuracy. Additionally,
$\bar{\Delta}_{T,v}$ exhibits a scaling relationship with the number
of antennas $N_{\mathrm{t}}$ as $\mathcal{O}(1/N_{\mathrm{t}}^{3})$,
indicating that increasing the antenna count at each \ac{ap}
provides substantial improvements in velocity estimation precision.

It is worth highlighting that the results do {\em not} assume LOS
or NLOS conditions. Specifically, regardless of a possibly large angular
variance $\sigma_{\theta}^{2}$ in NLOS, we can achieve arbitrarily
high accuracy in estimating the parameters of the trajectory given
a sufficient number of independent measurements. Moreover, a larger
$G_{1}/\sigma_{\mathrm{n}}^{2}$, which results from richer array
geometries or higher \ac{snr}, accelerates
the rate at which the \ac{crlb} for both
$\mathbf{x}$ and $\mathbf{v}$ decreases as $T$ increases.

\section{Algorithm Design}

\label{sec:Problem-Alg}

The previous section delivers two important messages: First, from
Theorem \ref{thm:LB-F-xv-un}, accumulating enough \ac{aod} measurements
can recover the trajectory even when the \ac{aod} is highly noisy
under NLOS. While the result was developed under an artificial model
for a simplified mobility, we can expect to recover at least a partial
trajectory given sufficient data. Second, from Theorem \ref{thm:Spatial-Continuity},
the pairwise \ac{csi} distance is consistent with the physical distance
of the corresponding positions in a small neighborhood under NLOS.
As a result, we can regulate the estimated trajectory by relating
the \ac{csi} distance with the physical distance, hoping to enhance
the estimation in the \ac{nlos} area. As indicated by Theorem \ref{thm:Spatial-Continuity},
we also need to identify the propagation conditions and only pair
measurements in the same propagation condition. Finally, we can exploit
the mobility model \eqref{eq:discretized-mobility-model} and the
Bayesian formulation \eqref{eq:p_Yt_Xt2} so that the estimated trajectory
in the \ac{nlos} area also benefit from the information collected
in the LOS area via the Bayesian chain rule.

In this section, we develop an algorithm framework to incorporate
all these design philosophy using Bayesian approaches.

\subsection{Feature Engineering}

\subsubsection{RSS Feature for LOS/NLOS Discrimination}

While we try not to rely on \ac{rss} for location signature, \ac{rss}
is still a good indicator for \ac{los} and \ac{nlos} discrimination
especially around the boundary of the two propagation regions. The
\ac{rss} $s_{t,q}$ at the $q$th \ac{ap} at time slot $t$ is simply
extracted as the channel power in logarithm scale:
\begin{equation}
s_{t,q}=10\log_{10}\|\mathbf{H}_{t,q}\|_{\mathrm{F}}^{2}.\label{eq:H-s}
\end{equation}

The \ac{rss} measurements are to be fitted to a conditional path
loss model that accounts for location and \ac{ap} dependent propagation
conditions: 
\begin{equation}
s_{t,q}=\beta_{q}^{(k)}+\alpha_{q}^{(k)}\log_{10}d(\mathbf{x}_{t},\mathbf{o}_{q})+\xi_{q}^{(k)},\quad k\in\{0,1\}\label{eq:rss-conditional-Gaussian}
\end{equation}
where $k$ indicates the propagation condition, with $k=0$ for \ac{los}
and $k=1$ for \ac{nlos}. Here, $\beta_{q}^{(k)}$ represents the
AP-dependent reference path loss, $\alpha_{q}^{(k)}$ is the AP-dependent
path loss exponent; and $d(\mathbf{x}_{t},\mathbf{o}_{q})=\|\mathbf{o}_{q}-\mathbf{x}_{t}\|_{2}$
denotes the Euclidean distance between the \ac{ap} at $\mathbf{o}_{q}$
and the mobile user at $\mathbf{x}_{t}$. The term $\xi_{q}^{(k)}\sim\mathcal{N}(0,\sigma_{s,q,k}^{2})$
is to model the randomness due to multipath fading, body shadowing,
and antenna pattern. The parameters $\alpha_{q}^{(k)}$, $\beta_{q}^{(k)}$,
$\sigma_{s,q,k}^{2}$ are to be jointly estimated from the data.

\subsubsection{AoD Feature}

We extract the \ac{aod} of the \emph{dominant} path from the \ac{mimo}-\ac{ofdm}
channel $\mathbf{H}_{t,q}$ using a subspace approach just as the
MUSIC algorithm \cite{ZhaQT:J95}. Specifically, perform eigen-decomposition
of the sample covariance matrix
\[
\mathbf{R}_{t,q}=\frac{1}{M}\mathbf{H}_{t,q}\mathbf{H}_{t,q}^{\mathrm{H}}
\]
and obtain the eigenvectors $\tilde{\mathbf{u}}_{t,q}^{(1)},\tilde{\mathbf{u}}_{t,q}^{(2)},\dots,\tilde{\mathbf{u}}_{t,q}^{(N_{\text{t}})}$
arranged in a decreasing order of the corresponding eigenvalues. Construct
the noise subspace matrix 
\[
\mathbf{U}_{t,q}=\left[\tilde{\mathbf{u}}_{t,q}^{(2)},\tilde{\mathbf{u}}_{t,q}^{(3)},\dots,\tilde{\mathbf{u}}_{t,q}^{(N_{t})}\right]\in\mathbb{C}^{N_{t}\times(N_{t}-1)}
\]
by skipping the dominant eigenvector $\tilde{\mathbf{u}}_{t,q}^{(1)}$.

Using the same principle of MUSIC algorithm, the relative \ac{aod}
$\varphi_{t,q}$ of the dominant path for position $\mathbf{x}_{t}$
\ac{wrt} the reference direction $\phi_{q}$ of the $q$th \ac{ap}
can be obtained by maximizing the following pseudo-spectrum 
\[
\hat{\varphi}_{t,q}=\underset{\varphi\in(-\pi/2,\pi/2)}{\mathrm{argmax}}\frac{1}{\mathbf{a}^{\mathrm{H}}(\varphi)\mathbf{U}_{t,q}\mathbf{U}_{t,q}^{\mathrm{H}}\mathbf{a}(\varphi)}
\]
where $\mathbf{a}(\varphi)$ is the steering vector given by (\ref{eq:steering}).
Thus, the \ac{aod} in the XOY coordinate system of the dominant path
from the \ac{ap} at $\mathbf{o}_{q}$ to position $\mathbf{x}_{t}$
is estimated as
\begin{equation}
\hat{\theta}_{t,q}=(\hat{\varphi}_{t,q}+\phi_{q})\text{ mod }2\pi.\label{eq:H-theta}
\end{equation}

The estimated \ac{aod} $\hat{\theta}_{t,q}$ is to be fitted to a
conditional Gaussian model 
\begin{equation}
\hat{\theta}_{t,q}\sim\mathcal{N}\left(\phi(\mathbf{x}_{t},\mathbf{o}_{q}),\sigma_{\theta,k}^{2}\right),\quad k\in\{0,1\}\label{eq:aod-conditional-Gaussian}
\end{equation}
where $\phi(\mathbf{x}_{t},\mathbf{o}_{q})$ defines the geometric
azimuth angle from the \ac{ap} location $\mathbf{o}_{q}$ to the
position $\mathbf{x}_{t}$. Likewise, $k=0$ stands for the \ac{los}
condition and $k=1$ stands for the \ac{nlos} condition, and the
parameters $\sigma_{\theta,k}^{2}$ are to be jointly fitted from
the data. {Experimental results in Figure
\ref{fig:AoDDelay} (b) verify that the variance in the LOS region
is significantly smaller than the variance in the NLOS region.}

\subsubsection{Delay Spread Feature for LOS/NLOS Discrimination}

\begin{figure}[t]
\centering{}\includegraphics[width=1\columnwidth]{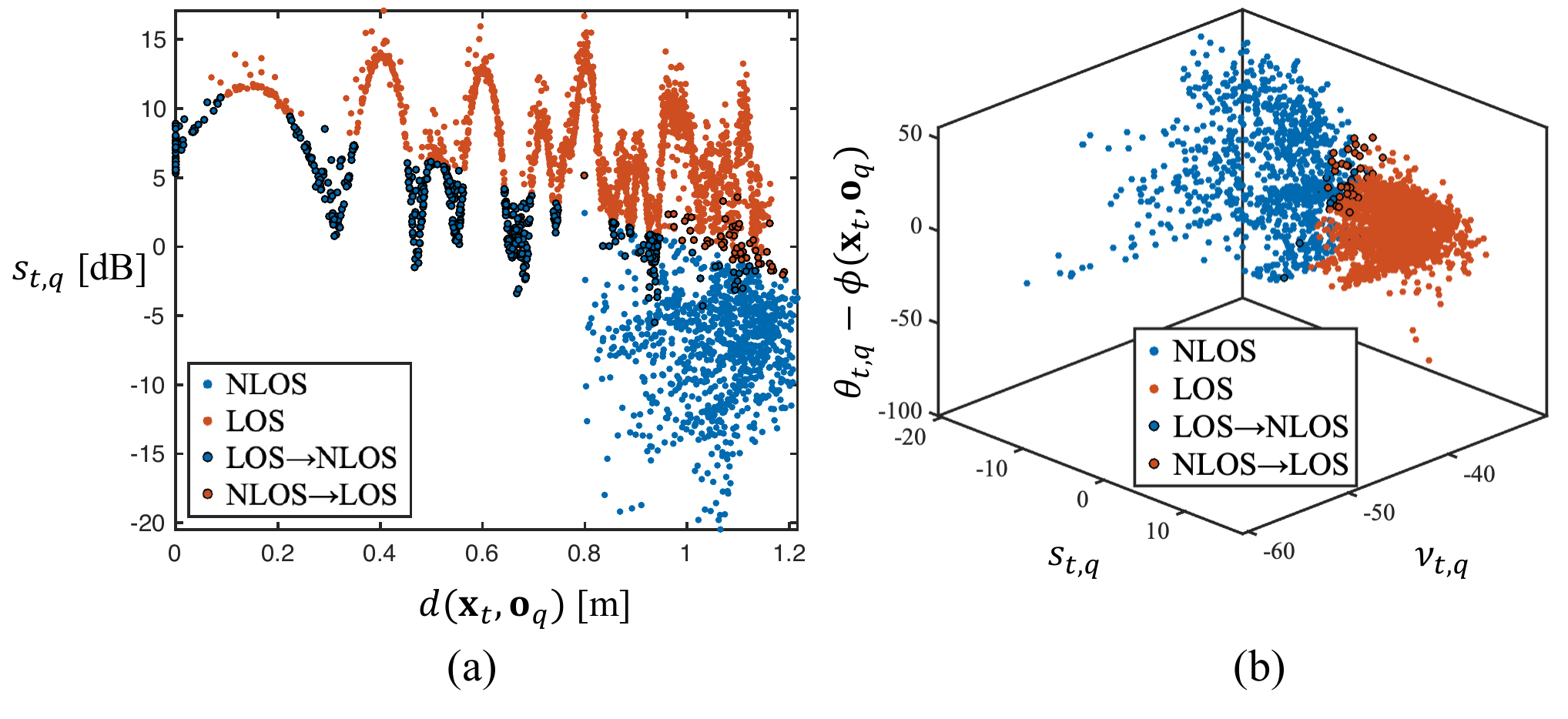}\caption{(a) With known user location, the RSS path-loss model is fitted to
perform LOS/NLOS identification, yielding an error rate of 20.5\%.
(b) By further leveraging the relationship among relative delay spread,
\ac{aod}, and \ac{rss}, the proposed method reduces the identification
error to 1.6\%. \label{fig:AoDDelay}}
\end{figure}

From the multipath channel model \eqref{eq:MIMO-OFDM channel model},
if there is only one dominant path, the magnitude of the frequency
domain channel is roughly constant, because the Fourier transform
of a delta function is constant. By contrast, if there are several
significant multipath components, the magnitude of the frequency domain
channel fluctuates. We thus construct a feature $\nu_{t,q}$ as follows

\begin{equation}
\nu_{t,q}=10\log_{10}\left(\text{Var}\left(\left|\frac{\mathbf{H}_{t,q}}{\|\mathbf{H}_{t,q}\|_{2}}\right|\right)\right)\label{eq:H-nu}
\end{equation}
to empirically capture how the channel energy spreads over the multipaths,
where $|\cdot|$ denotes the element-wise absolute value operation
and $\text{Var}(\cdot)$ denotes variance of the elements of the matrix
or vector.

The feature pair $(s_{t,q},\nu_{t,q})$ can be modeled using the following
Gaussian mixture model: 
\begin{equation}
[s_{t,q},\nu_{t,q}]^{\mathrm{T}}\sim\mathcal{N}\left(\mathbf{m}_{k},\bm{\Upsilon}_{k}\right),\quad k\in\{0,1\}\label{eq:delay-conditional-Gaussian}
\end{equation}
where $k=0$ corresponds to LOS case and $k=1$ to the NLOS case.
The parameters $\mathbf{m}_{k}$ and $\bm{\Upsilon}_{k}$ are jointly
estimated from empirical data. This Gaussian mixture model is motivated
by experimental observations, as illustrated in Figure \ref{fig:AoDDelay}(b).
Compared to relying solely on RSS for LOS/NLOS identification as shown
in Figure \ref{fig:AoDDelay}(a), the proposed feature demonstrates
a much better clustering behavior that may simplify the LOS/NLOS discrimination.

\subsubsection{CSI Distance Feature in NLOS Condition}

Define the CSI distance of two channels $\mathbf{H}_{i,q}$ and $\mathbf{H}_{j,q}$
based on (\ref{eq:Ru-hat}) and Remarks \ref{rem:Practical-Distance-metric}
and \ref{rem:Multiple-Antenna-Extension} as 
\begin{align*}
\hat{u}(\mathbf{H}_{i,q},\mathbf{H}_{j,q}) & =\underset{u\in\{0,1,...,M-1\}}{\arg\max}\;\Big|\frac{1}{N_{\text{t}}M}\sum_{n=1}^{N_{\text{t}}}\sum_{m=0}^{M-1}\Big[h_{i,q}^{(m,n)}\\
 & \qquad\qquad\qquad\qquad\times\left(h_{j,q}^{(m,n)}\right)^{*}\cdot e^{j\frac{2\pi mu}{M}}\Big]\Big|
\end{align*}
{where $h_{t,q}^{(m,n)}$ is the channel
of the $m$th subcarrier at the $n$th antenna.}

Denote $u_{t,q}^{(k)}\in\{0,1\}$ to indicate whether $\mathbf{H}_{t,q}$
is LOS, \emph{i.e.}, $u_{t,q}^{(0)}=1$ or NLOS, \emph{i.e.}, $u_{t,q}^{(1)}=1$.
Note that $u_{t,q}^{(0)}+u_{t,q}^{(1)}\equiv1$, and either $u_{t,q}^{(k)}=u_{\tau,q}^{(k)}$
for both $k=1,2$ or $u_{t,q}^{(k)}\neq u_{\tau,q}^{(k)}$ for both
$k=1,2$.

As an implication of Theorem \ref{thm:Spatial-Continuity}, the feature
$\hat{u}(\mathbf{H}_{i,q},\mathbf{H}_{j,q})$ is to be fitted to the
following regression model for every pair of measurements $i,j$ such
that they share the same NLOS propagation condition $u_{i,q}^{(1)}=u_{j,q}^{(1)}=1$,
\begin{equation}
\hat{u}(\mathbf{H}_{i,q},\mathbf{H}_{j,q})=\frac{B}{c}d(\mathbf{x}_{i},\mathbf{x}_{j})+\mathcal{N}(0,\sigma_{u}^{2})\label{eq:Dd-conditional-Gaussian}
\end{equation}
where the variance $\sigma_{u}^{2}$ quantifies the uncertainty. The
model \eqref{eq:Dd-conditional-Gaussian} describes the phenomenon
that the larger the physical distance $d$, the larger the CSI distance
$\hat{u}$, following the scaling $\hat{u}=\frac{B}{c}d$ inspired
from Theorem \ref{thm:Spatial-Continuity}.

The model \eqref{eq:Dd-conditional-Gaussian} was found to be consistent
with the data from the empirical studies in NLOS condition as shown
in {Figure~\ref{fig:csiDistance}(a).}

\subsection{A Bayesian Formulation}

\subsubsection{Gaussian Mixture Model}

Based on the feature derived in \eqref{eq:H-s}, \eqref{eq:H-theta},
and \eqref{eq:H-nu}, the observed radio signature can be constructed
as
\[
\mathbf{y}_{t,q}=\mathcal{H}(\mathbf{H}_{t,q})=[s_{t,q},\hat{\theta}_{t,q},s_{t,q},\nu_{t,q}]^{\mathrm{T}}.
\]
According to the conditional models \eqref{eq:rss-conditional-Gaussian},
\eqref{eq:aod-conditional-Gaussian}, and \eqref{eq:delay-conditional-Gaussian},
the observed radio signature vector $\mathbf{y}_{t,q}$ follows a
Gaussian mixture distribution 
\begin{align}
p(\mathbf{y}_{t,q}|\mathbf{x}_{t})=\sum_{k=0}^{1}\pi_{t,q}^{(k)}\mathcal{N}(\bm{\mu}^{(k)}(\mathbf{x}_{t}),\bm{\Sigma}^{(k)})\label{eq:prob-y}
\end{align}
where $\pi_{t,q}^{(k)}=\mathbb{P}\{u_{t,q}^{(k)}=1\}$ is the probability
that position $\mathbf{x}_{t}$ is at propagation condition $k$ for
the $q$th \ac{ap}, the mean vector is given by
\[
\bm{\mu}^{(k)}(\mathbf{x}_{t})=\left[\begin{array}{c}
\beta_{q}^{(k)}+\alpha_{q}^{(k)}\log_{10}d(\mathbf{o}_{q},\mathbf{x}_{t})\\
\phi(\mathbf{x}_{t},\mathbf{o}_{q})\\
\mathbf{m}_{k}
\end{array}\right]
\]
and the covariance matrix
\[
\bm{\Sigma}^{(k)}=\left[\begin{array}{ccc}
\sigma_{s,q,k}^{2} & 0 & 0\\
0 & \sigma_{\theta,k}^{2} & 0\\
0 & 0 & \bm{\Upsilon}_{k}
\end{array}\right].
\]

For the CSI distance model in \eqref{eq:Dd-conditional-Gaussian},
considering that the spatial continuity is only valid in a small neighborhood
and under NLOS conditions $k=1$. {For each
AP $q$, if the location $\mathbf{x}_{t}$ is in the NLOS region of
AP $q$, the NLOS neighborhood of $\mathbf{x}_{t}$, with respect
to all other position variables in $\mathcal{X}_{T}$, is defined
as
\[
\tilde{\mathcal{N}}_{q}(\mathbf{x}_{t})=\left\{ \tau:d(\mathbf{x}_{t},\mathbf{x}_{\tau})<\tilde{\delta},\ u_{t,q}^{(1)}=1,\ u_{\tau,q}^{(1)}=1\right\} ,
\]
where $\tilde{\delta}$ is a predefined distance threshold. }Define
$\mathcal{H}_{t}$ as a collection of the channel $\mathbf{H}_{t,q}$
up to time $t$ for all \acpl{ap} $q$. We can construct a Gaussian
log-likelihood function as
\begin{align}
f(\mathbf{H}_{t,q}|\mathcal{H}_{T},\mathcal{X}_{T}) & =\frac{1}{|\tilde{\mathcal{N}}_{q}(\mathbf{x}_{t})|}\sum_{\tau\in\tilde{\mathcal{N}}_{q}(\mathbf{x}_{t})}\log\Big[\frac{1}{\sqrt{2\pi\sigma_{u}^{2}}}\label{eq:fGtq}\\
\times & \exp\Big(-\frac{[\hat{u}(\mathbf{H}_{t,q},\mathbf{H}_{\tau,q})-\frac{B}{c}d(\mathbf{x}_{t},\mathbf{x}_{\tau})]^{2}}{2\sigma_{u}^{2}}\Big)\Big],\nonumber 
\end{align}
{where $|\tilde{\mathcal{N}}_{q}(\mathbf{x}_{t})|$
denotes the cardinality of the set $\tilde{\mathcal{N}}_{q}(\mathbf{x}_{t})$.
Moreover, $f(\mathbf{H}_{t,q}|\mathcal{H}_{T},\mathcal{X}_{T})$ is
set to zero if $\mathbf{x}_{t}$ is located in the LOS region of AP
$q$.}

\subsubsection{Spatially Regularized Likelihood Formulation}

Recall the joint probability $p(\mathcal{Y}_{T},\mathcal{X}_{T})$
in \eqref{eq:p_Yt_Xt2}. Denote $\bm{\Theta}_{\text{m}}$ as the collection
of all the parameters in the mobility model \eqref{eq:discretized-mobility-model}
and $\bm{\Theta}_{\text{p}}$ as the collection of all the remaining
parameters which are related to the radio signature, we construct
a regularized likelihood as follows
\[
\mathcal{L}(\mathcal{X}_{T},\bm{\Theta}_{\text{p}},\bm{\Theta}_{\text{m}})=\log p(\mathcal{Y}_{T},\mathcal{X}_{T})+\eta\sum_{t=1}^{T}\sum_{q=1}^{Q}f(\mathbf{H}_{t,q}|\mathcal{H}_{T},\mathcal{X}_{T})
\]
where $\eta>0$ is some weighting factor and the last term applies
regularization only to locations under the NLOS condition.

The trajectory inference problem can be formulated as
\begin{align}
\underset{\mathcal{X}_{T},\bm{\Theta}_{\mathrm{p}},\bm{\Theta}_{\mathrm{m}}}{\mathrm{maximize}} & \quad\mathcal{L}(\mathcal{X}_{T},\bm{\Theta}_{\mathrm{p}},\bm{\Theta}_{\mathrm{m}})\label{eq:P0}\\
\text{subject to} & \quad u_{t,q}^{(k)}\in\{0,1\},\;u_{t,q}^{(0)}+u_{t,q}^{(1)}=1\nonumber \\
 & \quad\mathbf{x}_{t}\in\mathcal{\mathcal{V}},\qquad t=1,2,\dots,T\nonumber \\
 & \quad(\mathbf{x}_{t},\mathbf{x}_{t-1})\in\mathcal{E},\qquad t=2,\dots,T.\nonumber 
\end{align}

\subsection{Algorithm Design}

To solve the joint trajectory inference and parameter estimation problem
in equation (\ref{eq:P0}), we observe that, given $\mathcal{X}_{T}$,
the variables $\bm{\Theta}_{\text{{p}}}$ and $\bm{\Theta}_{\text{{m}}}$
are decoupled. This is because the term $p(\mathbf{y}_{t,q}|\mathbf{x}_{t})$
in equation (\ref{eq:P0}) only depends on $\bm{\Theta}_{\text{{p}}}$,
while the term $\mathbb{P}(\mathbf{x}_{t}|\mathbf{x}_{t-1},\mathbf{x}_{t-2})$
only depends on $\bm{\Theta}_{\text{{m}}}$. Consequently, $\bm{\Theta}_{\text{{p}}}$
and $\bm{\Theta}_{\text{{m}}}$ can be solved through two parallel
subproblems derived from equation (\ref{eq:P0}), as follows:
\begin{align*}
(\mathrm{P1}):\underset{\bm{\Theta}_{\text{{m}}}}{\mathrm{maximize}} & \;\;\sum_{t=3}^{T}\log\mathbb{P}(\mathbf{x}_{t}|\mathbf{x}_{t-1},\mathbf{x}_{t-2};\bm{\Theta}_{\text{{m}}})\\
(\mathrm{P2}):\underset{\bm{\Theta}_{\text{{p}}}}{\mathrm{maximize}} & \;\;\sum_{t=1}^{T}\sum_{q=1}^{Q}\log\sum_{k=0}^{1}\pi_{t,q}^{(k)}\mathcal{N}(\bm{\mu}^{(k)}(\mathbf{x}_{t}),\bm{\Sigma}^{(k)})\\
 & \;\;\;\;+\eta\sum_{t=1}^{T}\sum_{q=1}^{Q}f(\mathbf{H}_{t,q}|\mathcal{H}_{T},\mathcal{X}_{T})\\
\text{subject to} & \quad u_{t,q}^{(k)}\in\{0,1\},\;u_{t,q}^{(0)}+u_{t,q}^{(1)}=1
\end{align*}

On the other hand, given the variables $\hat{\bm{\Theta}}_{\text{{p}}}$
and $\hat{\bm{\Theta}}_{\text{{m}}}$ as the solutions to (P1) and
(P2), respectively, the trajectory $\mathcal{X}_{T}$ can be solved
by: 
\begin{align*}
(\mathrm{P3}):\underset{\mathcal{X}_{T}}{\mathrm{maximize}} & \;\;\mathcal{L}(\mathcal{X}_{T},\bm{\Theta}_{\mathrm{p}},\bm{\Theta}_{\mathrm{m}})\\
\text{subject to} & \quad\mathbf{x}_{t}\in\mathcal{\mathcal{V}},\qquad t=1,2,\dots,T\\
 & \quad(\mathbf{x}_{t},\mathbf{x}_{t-1})\in\mathcal{E},\qquad t=2,\dots,T.
\end{align*}
This naturally leads to an alternating optimization strategy. In this
strategy, we solve for $\mathcal{X}_{T}$ from problem (P3), and then
for $\hat{\bm{\Theta}}_{\text{{p}}}$ and $\hat{\bm{\Theta}}_{\text{{m}}}$
from problems (P1) and (P2) iteratively. Since the corresponding iterations
never decrease the objective function in equation (\ref{eq:P0}),
which is bounded above, the iterations are guaranteed to converge.

\subsubsection{Solution to (P1) for the Mobility Model}

With given $\bm{\Theta}_{\text{{p}}}$ and $\mathcal{X}_{T}$, according
to the mobility model in (\ref{eq:mobility-model}), setting the derivative
of $p(\mathbf{x}_{t}|\mathbf{x}_{t-1},\mathbf{x}_{t-2};\bm{\Theta}_{\text{{m}}})$
\ac{wrt} $\bm{\Theta}_{\text{m}}=\{\bar{\mathbf{v}},\sigma_{\text{v}}^{2}\}$
to zero, we find that the corresponding solution

\begin{align}
\bar{\mathbf{v}} & =\frac{\sum_{t=3}^{T}(\mathbf{x}_{t}-(1+\gamma)\mathbf{x}_{t-1}+\gamma\mathbf{x}_{t-2})}{(T-2)(1-\gamma)\delta}\label{eq:solution-v}\\
\sigma_{\text{v}}^{2} & =\frac{\sum_{t=3}^{T}\|\mathbf{x}_{t}-(1+\gamma)\mathbf{x}_{t-1}+\gamma\mathbf{x}_{t-2}-(1-\gamma)\delta\bar{\mathbf{v}}\|_{2}^{2}}{2(T-2)\delta^{2}}\label{eq:solution-v-sigma}
\end{align}
is unique. Since (P1) is an unconstrained optimization problem, (\ref{eq:solution-v})--(\ref{eq:solution-v-sigma})
give the optimal solution to (P1).

\subsubsection{Solution to (P2) via Expectation Maximization}

\label{subsec:Solution-to-(P2)}

Given the known parameters $\bm{\Theta}_{\text{m}}$ and the trajectory
$\mathcal{X}_{T}$, we can solve ($\mathrm{P2}$) for the set of parameters
$\bm{\Theta}_{\text{p}}=\{\beta_{q}^{(k)},\alpha_{q}^{(k)},\mathbf{m}_{k},\sigma_{s,q,k}^{2},\sigma_{\theta,k}^{2},\bm{\Upsilon}_{k},u_{t,q}^{(k)},\sigma_{u}^{2}\}$.
This is a typical maximum likelihood estimation problem with latent
variables, which can be efficiently solved using the Expectation-Maximization
(EM) algorithm \cite{MonTod:J96}, where one may iteratively estimate
the posterior probability $c_{t,q}^{(k)}$ and the model parameter
$\Xi=\{\pi_{t,q}^{(k)},\bm{\mu}^{(k)}(\mathbf{x}_{t}),\bm{\Sigma}^{(k)}\}$.

Specifically, given a random initialization $\Xi$, the posterior
probability $c_{t,q}^{(k)}$ is calculated by
\begin{align}
c_{t,q}^{(k)} & =\frac{\pi_{t,q}^{(k)}\mathcal{N}(\mathbf{x}_{t};\bm{\mu}^{(k)}(\mathbf{x}_{t}),\bm{\Sigma}^{(k)})}{\sum_{j=0}^{1}\pi_{t,q}^{(j)}\mathcal{N}(\mathbf{x}_{t};\bm{\mu}^{(j)}(\mathbf{x}_{t}),\bm{\Sigma}^{(j)})}.\label{eq:ck}
\end{align}

With the updated $c_{t,q}^{(k)}$, we update $\Xi$ as follows. For
updating the parameters $\{\beta_{q}^{(k)},\alpha_{q}^{(k)},\mathbf{m}_{k}\}$
in $\bm{\mu}^{(k)}(\mathbf{x}_{t})$, let $\mathbf{D}_{q}\in\mathbb{R}^{T\times2}$
be a matrix with the $t$th row $[1,\log_{10}d(\mathbf{o}_{q},\mathbf{x}_{1})]$
and $\mathbf{w}_{q}^{(k)}=[c_{1,q}^{(k)},\ldots,c_{T,q}^{(k)}]^{\mathrm{T}}$
be a vector. Then, the mean parameters $\mathbf{m}_{k}$ is updated
by $\mathbf{m}_{k}=\sum_{t,q}c_{t,q}^{(k)}[s_{t,q},\nu_{t,q}]^{\mathrm{T}}/(\sum_{t,q}c_{t,q}^{(k)})$,
and $[\beta_{q}^{(k)},\alpha_{q}^{(k)}]^{\mathrm{T}}$ are estimated
using weighted least squares
\begin{align*}
[\beta_{q}^{(k)},\alpha_{q}^{(k)}]^{\mathrm{T}} & =\left(\mathbf{D}_{q}^{\mathrm{T}}\mathrm{Diag}(\mathbf{w}_{q}^{(k)})\mathbf{D}_{q}\right)^{-1}\mathbf{D}_{q}^{\mathrm{T}}\mathrm{Diag}(\mathbf{w}_{q}^{(k)})\mathbf{s}_{q}
\end{align*}
where $\mathrm{Diag}(\mathbf{w}_{q}^{(k)})$ denotes the diagonal
matrix whose diagonal entries are given by the elements of the vector
$\mathbf{w}_{q}^{(k)}$.

To update the variances $\{\sigma_{s,q,k}^{2},\sigma_{\theta,k}^{2},\bm{\Upsilon}_{k}\}$
in $\bm{\Sigma}^{(k)}$, we leverage its diagonal structure and estimate
each component independently
\begin{align*}
\sigma_{s,q,k}^{2} & =\frac{\sum_{t,q}c_{t,q}^{(k)}\left(s_{t,q}-\beta_{q}^{(k)}-\alpha_{q}^{(k)}\log_{10}d(\mathbf{o}_{q},\mathbf{x}_{t})\right)^{2}}{\sum_{t,q}c_{t,q}^{(k)}}\\
\sigma_{\theta,k}^{2} & =\frac{\sum_{t,q}c_{t,q}^{(k)}\left(\hat{\theta}_{t,q}-\phi(\mathbf{x}_{t},\mathbf{o}_{q})\right)^{2}}{\sum_{t,q}c_{t,q}^{(k)}}\\
\bm{\Upsilon}_{k} & =\frac{\sum_{t,q}c_{t,q}^{(k)}\left([s_{t,q},\nu_{t,q}]^{\mathrm{T}}-\mathbf{m}_{k}\right)\left([s_{t,q},\nu_{t,q}]^{\mathrm{T}}-\mathbf{m}_{k}\right)^{\mathrm{T}}}{\sum_{t,q}c_{t,q}^{(k)}}.
\end{align*}
Then, the mixture weights are updated as $\pi_{t,q}^{(k)}=\frac{1}{TQ}\sum_{t=1}^{T}\sum_{q=1}^{Q}c_{t,q}^{(k)}$.

With the updated $\Xi$, we compute $c_{t,q}^{(k)}$ using \eqref{eq:ck}
again. Thereafter, a new round of updating for $\Xi$ can be started.
Let $\mathcal{L}^{(i)}$ denote the log-likelihood at iteration $i$.
The EM algorithm is terminated once the convergence condition is satisfied
$|\mathcal{L}^{(i+1)}-\mathcal{L}^{(i)}|<10^{-6}$. Note that the
spatial regularization is only meaningful in the NLOS case and is
therefore not considered in the LOS/NLOS discrimination procedure.

Finally, the LOS/NLOS assignment is given by
\[
u_{t,q}^{(k)}=\begin{cases}
1, & \text{if }k=\arg\max_{j}\:c_{t,q}^{(j)}\\
0, & \text{otherwise}.
\end{cases}
\]
and the regularization variance $\sigma_{u}^{2}$ is calculated as
\begin{align*}
\sigma_{u}^{2} & =\frac{1}{{\displaystyle \sum_{t=1}^{T}\sum_{q=1}^{Q}|\tilde{\mathcal{N}}_{q}(\mathbf{x}_{t})|}}\sum_{t=1}^{T}\sum_{q=1}^{Q}\sum_{\tau\in\tilde{\mathcal{N}}_{q}(\mathbf{x}_{t})}\Big[\hat{u}(\mathbf{H}_{t,q},\mathbf{H}_{\tau,q})\\
 & \qquad\qquad-\frac{B}{c}d(\mathbf{x}_{t},\mathbf{x}_{\tau})\Big]^{2}
\end{align*}
where $|\tilde{\mathcal{N}}_{q}(\mathbf{x}_{t})|$ denotes the cardinality
of the set $\tilde{\mathcal{N}}_{q}(\mathbf{x}_{t})$.

\subsubsection{Solution to (P3) for Trajectory Optimization}

\label{subsec:Solution-to-(P3)}

Problem (P3) searches for a trajectory in a discrete space that maximizes
the log-likelihood $\mathcal{L}(\mathcal{X}_{T},\bm{\Theta}_{\mathrm{p}},\bm{\Theta}_{\mathrm{m}})$
given the signal propagation parameters $\hat{\bm{\Theta}}_{\text{{p}}}$
and mobility model parameters $\hat{\bm{\Theta}}_{\text{{m}}}$. Problem
(P3) follows a classical \ac{hmm} optimization form, and can be efficiently
solved using a modified version of the Viterbi algorithm with globally
optimal guarantee.

At each step, there are $|\mathcal{V}|$ candidate locations considered,
but states with very low probabilities $\prod_{q=1}^{Q}p(\mathbf{y}_{t,q}|\mathbf{x}_{t};\hat{\bm{\Theta}}_{\text{p}})$
are highly unlikely to contribute to the optimal path. To improve
efficiency, states with probabilities below a threshold $\zeta$ are
pruned. Mathematically, this corresponds to retaining only the top
$n_{t}(\zeta)$ most probable locations at time slot $t$, where $n_{t}(\zeta)$
is the number of elements in the set $\{\mathbf{x}_{t}\mid\prod_{q=1}^{Q}p(\mathbf{y}_{t,q}|\mathbf{x}_{t};\hat{\bm{\Theta}}_{\text{p}})>\zeta,\mathbf{x}_{t}\in\mathcal{V}\}$.
Denote the maximum number of element in the set $n_{\mathrm{max}}(\zeta)=\max_{t}\{n_{t}(\zeta)\}$.

Considering the number of candidate previous states for the current
state, which is constrained by the graph structure, it is of the order
$\mathcal{O}(\varrho^{2}(D_{\text{m}}))$ for a square region, where
the max hop $\varrho(D_{\text{m}})$ at each step is determined by
$D_{\text{m}}$. Thus, the computational complexity of solving problem
(P3) is upper bounded by $\mathcal{O}(Tn_{\mathrm{max}}(\zeta)\varrho^{2}(D_{\text{m}}))$.

The overall algorithm is summarized in Algorithm \ref{alg:alternative-opt}.
We first initialize the propagation parameter $\bm{\Theta}_{\text{{p}}}$
and and the mobility parameter $\bm{\Theta}_{\text{{m}}}$ randomly
and then begin the alternating update of $\mathcal{X}_{T}$, $\bm{\Theta}_{\text{{p}}}$
and $\bm{\Theta}_{\text{{m}}}$ alternatively until convergence. Since
each iteration of this procedure never decreases the objective function,
which is bounded above, the iterative process is therefore guaranteed
to converge{.}

\begin{algorithm}
Initialize the parameter $\bm{\Theta}_{\text{{p}}}^{(0)}$, $\bm{\Theta}_{\text{{m}}}^{(0)}$
randomly.

Loop for the ($i+1$)th iteration:
\begin{itemize}
\item Update $\mathcal{X}_{T}^{(i+1)}$ using the method in Section \eqref{subsec:Solution-to-(P3)}.
\item Update $\bm{\Theta}_{\text{{p}}}^{(i+1)}$ using method in Section
\eqref{subsec:Solution-to-(P2)}.
\item Update $\bm{\Theta}_{\text{{m}}}^{(i+1)}$ using (\ref{eq:solution-v})-(\ref{eq:solution-v-sigma}).
\end{itemize}
Until $\mathcal{X}_{T}^{(i+1)}=\mathcal{X}_{T}^{(i)}$.

\caption{An alternating optimization algorithm for trajectory inference.\label{alg:alternative-opt}}
\end{algorithm}

\section{Numerical Experiments}

\label{sec:Experiments}

In this section, we first present the experimental setup and scenarios
in Section{~\ref{subsec:Datasets}}, followed
by a numerical validation of the theoretical results in Section{~\ref{subsec:Numerical-Validation}}.
Finally, we evaluate the accuracy of trajectory inference and the
constructed radio map in Section{~\ref{subsec:Trajectory-Recovery-Performance}}.

\subsection{Environmental Setup and Scenarios}

\label{subsec:Datasets}

This paper validates the proposed algorithm using two datasets:

\textbf{Synthetic Dataset I:} We simulate a trajectory of length $100$
meters using the mobility model defined in (\ref{eq:mobility-model}),
parameterized by $\gamma=1$, $\mathbf{v}=[1,0]^{\mathrm{T}}$ m/s,
$\mathbf{x}=[0,0]^{\mathrm{T}}$ m, and $\delta=0.1$ s. We consider
\ac{ap} at a height of 3 meters and a mobile user at a height of
1.5 meters equipped with $N_{\mathrm{t}}$ antennas. Two scenarios
are considered: in Scenario 1 (\ac{ap} deployed in a limited region,
c.f., Section IV-B), the number of \acpl{ap} surrounding the trajectory
is fixed at $Q=4$, 8, 12, 16, 20; in Scenario 2 (\ac{ap} deployed
in an unlimited region, c.f., Section III-C), the \acpl{ap} in the
target area follow a \ac{ppp} with densities $\kappa=0.64\times10^{-2}$,
$2.55\times10^{-2}$, $3.02\times10^{-2}$, $5.02\times10^{-2}$,
$7.02\times10^{-2}$, $9.02\times10^{-2}$, $1.02\times10^{-1}$,
and $2.55\times10^{-1}$ units per $\mathrm{m}^{2}$. The mobile user
can only detect \acpl{ap} within a radius of $R=$10, 20, 30, 40,
50, 60, 100 meters. The number of antennas in each \ac{ap} is $N_{\mathrm{t}}=2,4,8,16,32,64,128$.

\textbf{Synthetic Dataset} \textbf{II:} We utilized Wireless Insite{$^{\circledR}$}
to simulate a 26 m $\times$ 24 m indoor environment with a 264 m$^{2}$
area. As illustrated in Figure \ref{fig:antenna}, four APs with a
height of 3 meters were manually deployed at the corners of the room.
Each \ac{ap} is equipped with an 8-antenna omnidirectional \ac{ula}
array and configured with $M=64$ subcarriers using a \ac{mimo}-\ac{ofdm}
model. {We recorded the }\ac{csi}{
at receivers positioned at a height of 1.5 meters along a trajectory
with length of 167\,m. The sampling interval is set to $\delta=0.2$\,s}

\subsection{Numerical Validation of the Theoretical Results}

\label{subsec:Numerical-Validation}

\begin{figure}
\centering{}\includegraphics[width=1\columnwidth]{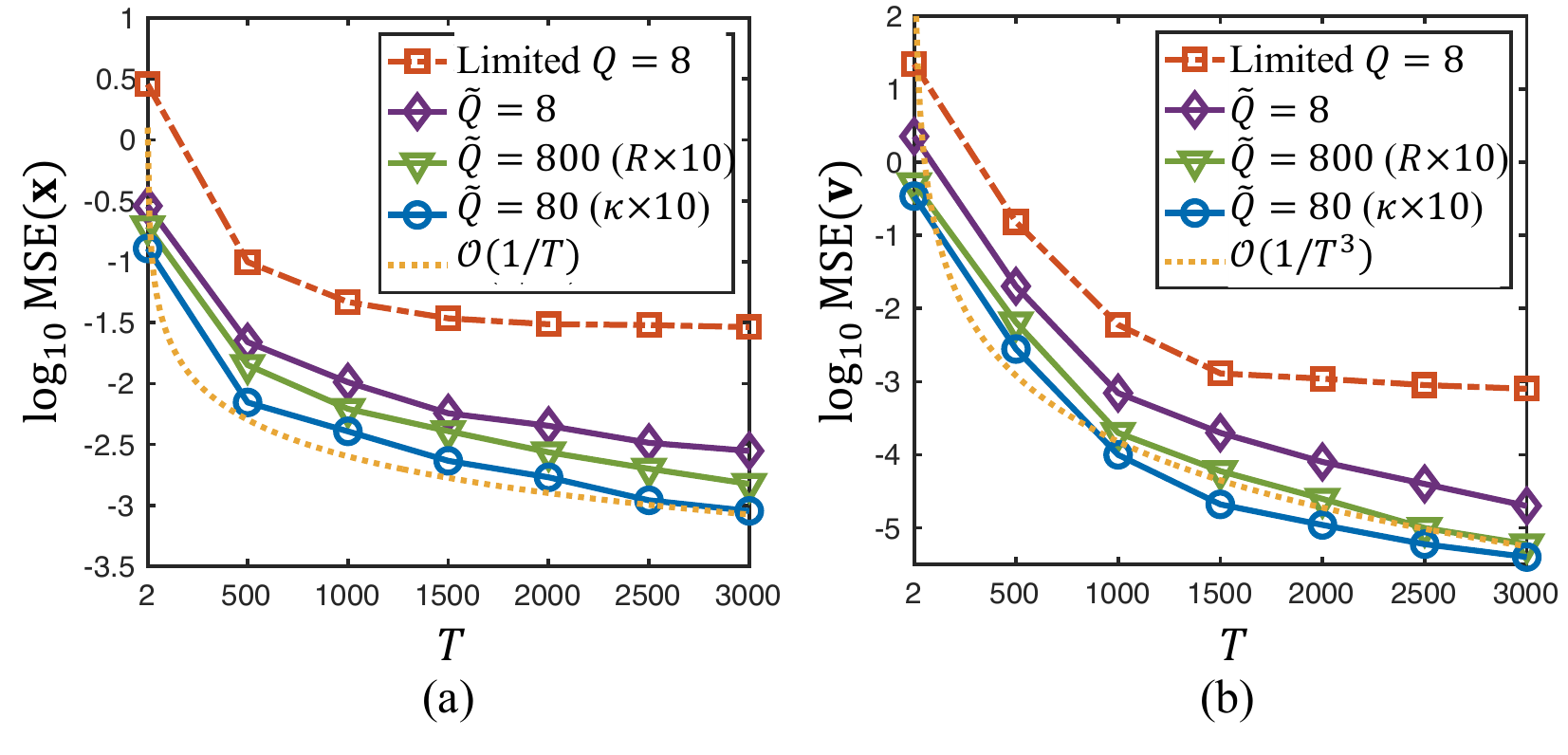}\vspace{-0.1in}
\caption{\ac{mse} of (a) $\mathbf{x}$ and (b)\textbf{ $\mathbf{v}$} with
different sample number $T$, the number of \ac{ap} $Q$, radius
$R$, and density $\kappa$.\label{fig:vx-T}}
\vspace{-0.1in}
\end{figure}

Figure \ref{fig:vx-T} illustrates the \ac{mse} defined as MSE($\mathbf{x}$)$=\|\mathbf{x}-\hat{\mathbf{x}}\|_{2}^{2}$
m$^{2}$ and MSE($\mathbf{x}$)$=\|\mathbf{v}-\hat{\mathbf{v}}\|_{2}^{2}$
m$^{2}$/s$^{2}$ on the synthetic dataset I with the parameter $\sigma_{\theta}=0.1$,
where $\hat{\mathbf{x}}$ and $\hat{\mathbf{v}}$ are the outputs
of the proposed algorithm.

In Scenario 1 of the synthetic dataset I with $Q=8$, the \ac{mse}
of $\mathbf{x}$ and $\mathbf{v}$ decreases as $T$ increases within
a limited region but does not converge to zero even when $T=20000$
in our experiments. This behavior is consistent with Proposition \ref{prop:LB-F-x}
and Proposition \ref{prop:LB-F-v}.

In Scenario 2 of the synthetic dataset I, we set $R=10$ m and $\kappa=2.55\times10^{-2}$
units per m$^{2}$, resulting in $\text{\ensuremath{\tilde{Q}}}\approx8$.
As $T$ increases, the rate at which MSE($\mathbf{x}$) decreases
follows $\mathcal{O}(1/T)$, and the rate at which MSE($\mathbf{v}$)
decreases follows $\mathcal{O}(1/T^{3})$, which is consistent with
Theorem \ref{thm:LB-F-xv-un}. Scenario 2 with $\text{\ensuremath{\tilde{Q}}}\approx8$
achieves a lower \ac{mse} compared to Scenario 1 with $Q=8$. Furthermore,
The MSE($\mathbf{x}$) and MSE($\mathbf{v}$) for the curves $\tilde{Q}=8$
and $\tilde{Q}=800$ in Figure \ref{fig:vx-T} both reach zero when
$T>3200$. We found that increasing $R$ results in a lower \ac{mse}
than increasing $\kappa$ under the same number of \acpl{ap}. This
is because $\tilde{\Delta}_{T,x}$ and $\tilde{\Delta}_{T,v}$ in
Theorem \ref{thm:LB-F-xv-un} is related to $\mathcal{O}(1/(\kappa\pi(r_{0}^{-2}-R^{-2})))$.
In addition, we found that increasing the radius $R$ from 50 to 500
meters results in a lower \ac{mse}, and increasing the density $\kappa$
from $2.55\times10^{-2}$ to $2.55\times10^{-1}$ also yields a lower
\ac{mse}.

\begin{figure}
\centering{}\includegraphics[width=1\columnwidth]{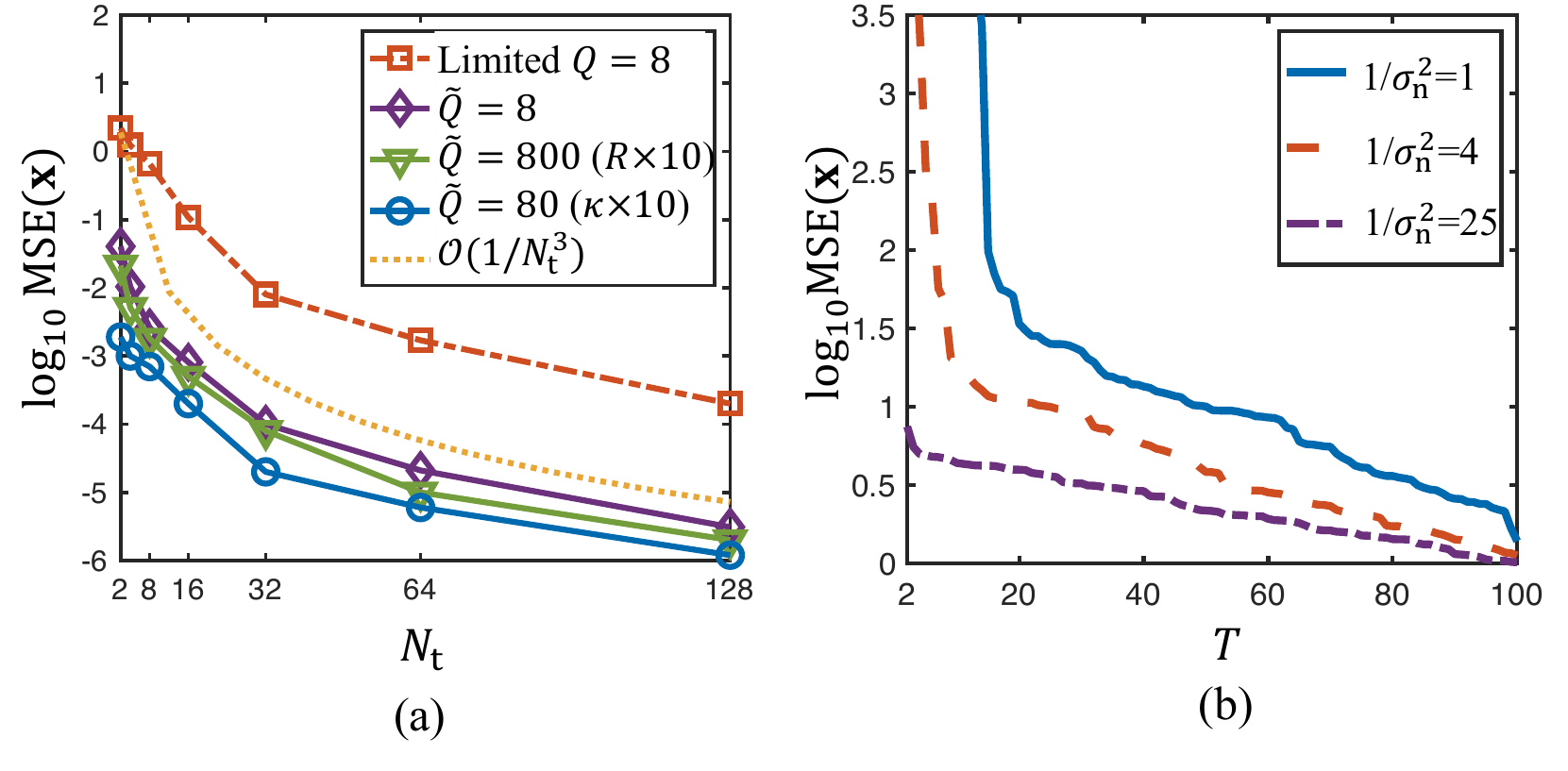}\vspace{-0.1in}
\caption{(a) The relationship between MSE($\mathbf{x}$) and the number of
antennas $N_{\mathrm{t}}$. (b) MSE($\mathbf{x}$) under different
noise $\sigma_{\mathrm{n}}^{2}$. \label{fig:Q-sigma}}
\vspace{-0.1in}
\end{figure}

We investigate the effect of the number of antenna $N_{\mathrm{t}}$
on MSE($\mathbf{x}$). The trajectory length is set to 100 m. As shown
in Figure \ref{fig:Q-sigma} (a), as $N_{\mathrm{t}}$ increases,
the rate at which MSE($\mathbf{x}$) decreases follows $\mathcal{O}(1/N_{\mathrm{t}}^{3})$,
which is consistent with Theorem \ref{thm:LB-F-xv-un}.

We also investigate the effect of the noise variance $\sigma_{\mathrm{n}}^{2}$
under the unlimited scenario, with $N_{\mathrm{t}}=2$, $R=10$ m,
$\kappa=2.55\times10^{-2}$ per m$^{2}$, and a trajectory length
of $500$ m. We consider $1/\sigma_{\mathrm{n}}^{2}=1,4,25$ for all
\acpl{ap}. As shown in Figure \ref{fig:Q-sigma}(b), a larger $\sigma_{\mathrm{n}}^{2}$
results in a faster convergence rate. Recall that $1/\sigma_{\mathrm{n}}^{2}$
is proportional to the \ac{snr}. Thus, a smaller $1/\sigma_{\mathrm{n}}^{2}$
leads to a slower decrease in the \ac{crlb} of $\mathbf{x}$ and
$\mathbf{v}$ as $T$ increases, as stated in Theorem \ref{thm:LB-F-xv-un}.

\begin{figure}[t]
\centering{}\includegraphics[width=1\columnwidth]{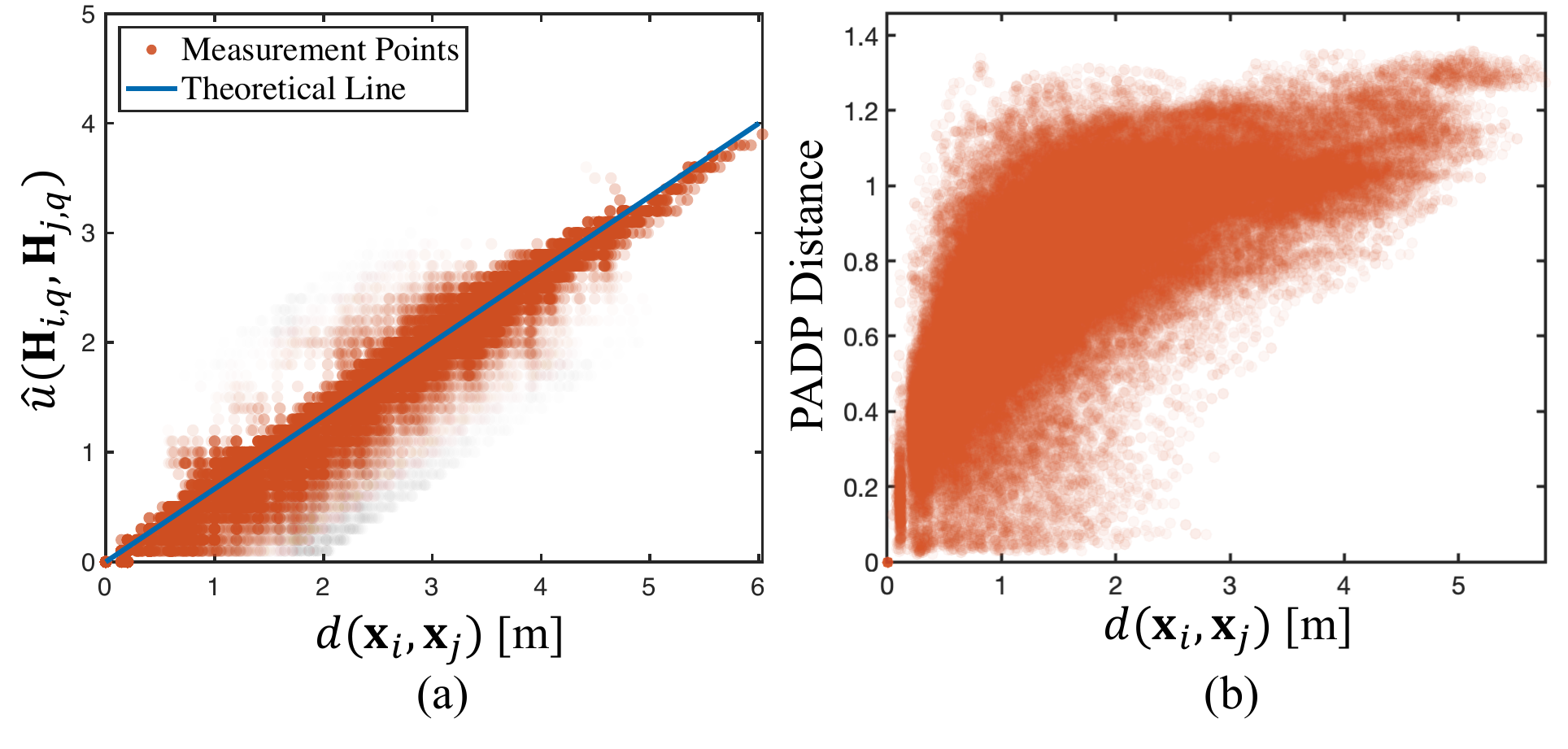}\vspace{-0.1in}
\caption{(a) The{ relationship between $\hat{u}(\mathbf{H}_{i,q},\mathbf{H}_{j,q})$
and $d(\mathbf{x}_{i},\mathbf{x}_{j})$ in NLOS region}. (b) The{
relationship between PADP distance and $d(\mathbf{x}_{i},\mathbf{x}_{j})$
in NLOS region}. \label{fig:csiDistance}}
\vspace{-0.1in}
\end{figure}

Figure \ref{fig:csiDistance}(a) illustrates the data point pairs
$\big(d(\mathbf{x}_{i},\mathbf{x}_{j}),\hat{u}(\mathbf{H}_{i,q},\mathbf{H}_{j,q})\big)$
along with the reference line $y=\frac{B}{c}x$. The data points exhibit
a good fit to the line, with a variance of 0.2, thereby verifying
the accuracy of Theorem \ref{thm:Spatial-Continuity}. {As
for a benchmark, we also analyze the power-angular-delay profile (PADP)
distance to capture spatial correlations between locations $\mathbf{x}_{t}$.
Let $\mathbf{H}_{t,q}\in\mathbb{C}^{N_{t}\times M}$ denote the channel
matrix. A dictionary $\mathbf{D}\in\mathbb{C}^{N_{t}\times N}$ with
$N=8N_{t}$ columns is constructed, and a normalized DFT matrix $\mathbf{F}_{M}\in\mathbb{C}^{M\times M}$
is defined with unit-norm columns. The normalized PADP is computed
as $\mathbf{G}_{t,q}=|\mathbf{D}^{\mathrm{H}}\mathbf{H}_{t,q}\mathbf{F}_{M}^{\mathrm{H}}|/\|\mathbf{D}^{\mathrm{H}}\mathbf{H}_{t,q}\mathbf{F}_{M}^{\mathrm{H}}\|_{\mathrm{F}}$,
projecting $\mathbf{H}_{t,q}$ onto the spatial-delay domain and normalizing
by its Frobenius norm. }Figure \ref{fig:csiDistance}(b){
shows the PADP distance $\|\mathbf{G}_{i}-\mathbf{G}_{j}\|_{\mathrm{F}}$
and the physical distance between two samples. It is evident that
the simple PADP distance does not exhibit a clear relationship with
physical distance.}

\subsection{Trajectory Inference Performance}

\label{subsec:Trajectory-Recovery-Performance}

\begin{table}[t]
\centering{}
\centering{}\caption{Comparison of average localization error ($E_{\mathrm{loc}}$) on
the dataset II. \label{tab:traj-performance}}
\begin{tabular}{l|>{\centering}p{0.97cm}>{\centering}p{0.97cm}>{\centering}p{0.97cm}>{\centering}p{0.97cm}>{\centering}p{0.97cm}}
\hline 
 & WCL \cite{WanUrr:J11} & AoDL\cite{GarWym:J18} & TDoAL\cite{SybHo:J18} & HAT\cite{AhaKal:C23} & CC \cite{TanPal:J25}\tabularnewline
\hline 
NLOS & 4.95 & 4.31 & 6.74 & 3.98 & 3.91\tabularnewline
Single LOS & 4.24 & 3.82 & 4.81 & 3.20 & 3.14\tabularnewline
Double LOS & 3.01 & 2.95 & 4.14 & 2.63 & 1.75\tabularnewline
All & 3.55 & 3.38 & 4.72 & 2.97 & 2.19\tabularnewline
\hline 
 & HRE\cite{XinChe:C24} & Proposed ($\eta=0$) & Proposed & GMA (Ideal) & \tabularnewline
NLOS & 3.31 & 1.56 & 1.07 & 1.02 & \tabularnewline
Single LOS & 2.51 & 1.21 & 0.82 & 0.79 & \tabularnewline
Double LOS & 1.26 & 0.91 & 0.59 & 0.57 & \tabularnewline
All & 1.64 & 1.02 & 0.68 & 0.66 & \tabularnewline
\hline 
\end{tabular}\vspace{-0.1in}
\end{table}

We evaluate the trajectory inference performance
of the proposed method using the average localization error, defined
as $E_{\mathrm{loc}}=\frac{1}{T}\sum_{t=1}^{T}\|\mathbf{x}_{t}-\hat{\mathbf{x}}_{t}\|_{2}$,
where $\mathbf{x}_{t}$ denotes the true data collection location
at time slot $t$, and $\hat{\mathbf{x}}_{t}$ is the corresponding
estimated location. The proposed approach is compared against six
baseline methods and two variants. The baselines include: (i) Weighted
Centroid Localization (WCL)~\cite{WanUrr:J11}, which estimates $\hat{\mathbf{p}}_{t}=\sum_{q=1}^{Q}w_{t,q}\mathbf{o}_{q}$
with weights $w_{t,q}=10^{s_{t,q}/20}/\sum_{l=1}^{Q}10^{s_{t,l}/20}$
derived from the received signal power $s_{t,q}$; (ii) AoD-Based
Localization (AoDL)~\cite{GarWym:J18}, which applies geometric triangulation
using \ac{aod} measurements and known
\ac{ap} coordinates; (iii) TDoA-Based
Localization (TDoAL)~\cite{SybHo:J18}, which estimates position
through the intersection of hyperboloids formed by \ac{tdoa}
measurements; (iv) Hybrid AoD-TDoA (HAT)~\cite{AhaKal:C23}, which
combines AoD and TDoA measurements to improve robustness; (v) Channel
Charting (CC)~\cite{TanPal:J25}, which maps \ac{csi}
features to physical space via a bilateration loss and line-of-sight
bounding-box regularization; and (vi) HMM-based RSS Embedding (HRE)~\cite{XinChe:C24},
which employs a graph-based hidden Markov model to infer trajectories
from \ac{rss}. In addition, we consider
two variants of our approach: the proposed method without the spatial
continuity constraint ($\eta=0$), and a Genius-aided Map-Assisted
(GMA) variant that assumes perfect knowledge of propagation parameters
and alternately updates the mobility model and trajectory, serving
as an upper performance bound. For the proposed method, we set the
mobility model trade-off parameter to $\gamma=0.5$, the regularization
parameter to $\eta=3000$, and the location space resolution in the
graph $\mathcal{G}$ to 0.2~m. The parameter $D_{\mathrm{m}}$ is
set to $12.4\delta$~m, where $12.4$~m/s corresponds to the maximum
human walking speed. We set $\tilde{\delta}=2$ m.

\begin{figure}[t]
\centering{}\includegraphics[width=1\columnwidth]{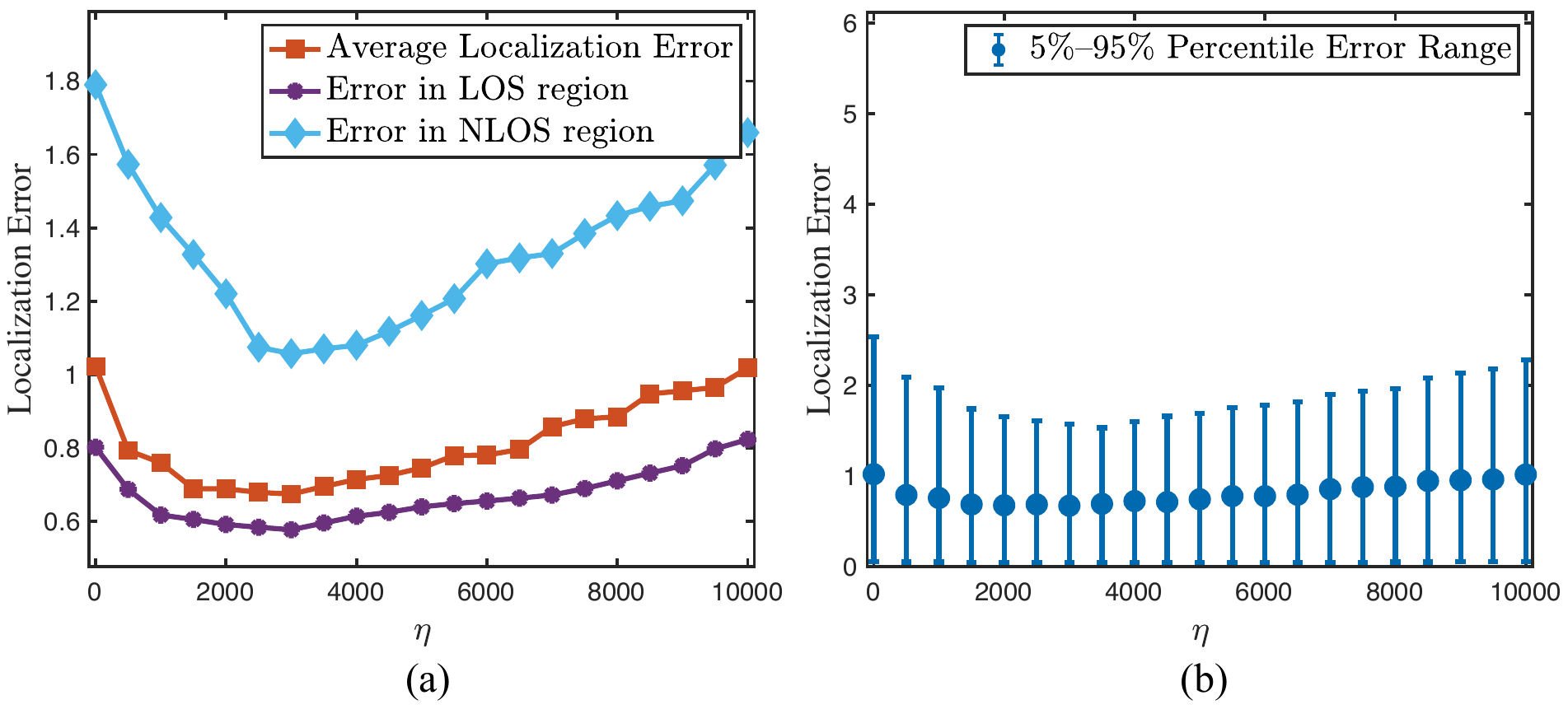}\vspace{-0.1in}
\caption{(a) The relationship between average localization error $E_{\mathrm{loc}}$
and the spatial regularization parameter $\eta$. (b) Average localization
error along with minimum and maximum errors as the parameter $\eta$.
\label{fig:errorVSeta}}
\vspace{-0.1in}
\end{figure}

As shown in Table \ref{tab:traj-performance}, the proposed method
achieves the best overall performance, with an average localization
error of just 0.68 meters, significantly outperforming all baselines.
Among the benchmarks, the time-delay-based method TDoAL performs poorly
across all scenarios, with an average error of 4.72 meters. The RSS-based
WCL method also yields unsatisfactory results, exhibiting an average
error of 3.55 meters, primarily due to the strong fluctuations in
signal strength caused by environmental factors. The angle-based AoDL
approach delivers slightly better results but still fails to provide
reliable localization, as angle information in NLOS regions cannot
ensure accurate positioning. The hybrid AoD-TDoA method HAT demonstrates
improved performance over both AoDL and TDoAL by combining angular
and delay information, achieving an average error of 2.97 meters.
However, it still suffers in NLOS environments, where multipath-induced
angular deviations significantly degrade its accuracy, resulting in
an error of 3.98 meters in the NLOS region. While CC demonstrates
moderately improved performance compared to HAT, it remains fundamentally
constrained in NLOS-dominated environments due to unresolved multipath
interference limitations. In contrast, the superior accuracy of our
proposed method stems from its integrated exploitation of power, angle,
and delay information, enabling a more holistic characterization of
the multipath propagation process. This joint modeling allows the
system to better distinguish between \ac{los} and \ac{nlos} scenarios
and enhances localization robustness in complex environments. In addition,
the proposed method with $\eta=0$ demonstrates that incorporating
spatial regularization constraints from the trajectory optimization
framework results in a 33.3\% reduction in error (decreasing from
1.02 to 0.68 meters), thereby quantitatively validating the necessity
of geometric consistency enforcement in trajectory inference.

Figure \ref{fig:errorVSeta} presents the average localization error
as a function of the regularization parameter $\eta$. When $\eta=0$,
i.e., without using the regularization term, the error reaches its
maximum. As $\eta$ increases, the error gradually decreases, reaching
its minimum value of 0.68 when $\eta=3000$. However, further increasing
$\eta$ leads to a rise in error. Nonetheless, any $\eta>0$ leads
to an improvement in localization accuracy due to the regularization
term. The effectiveness of the regularization term is particularly
pronounced in the NLOS regions, where it significantly reduces the
estimation error. This is also evidenced by the reduction in the maximum
error as $\eta$ increases (for $\eta<3000$), since the maximum error
mainly originates from NLOS regions.

\begin{figure}[t]
\centering{}\includegraphics[width=1\columnwidth]{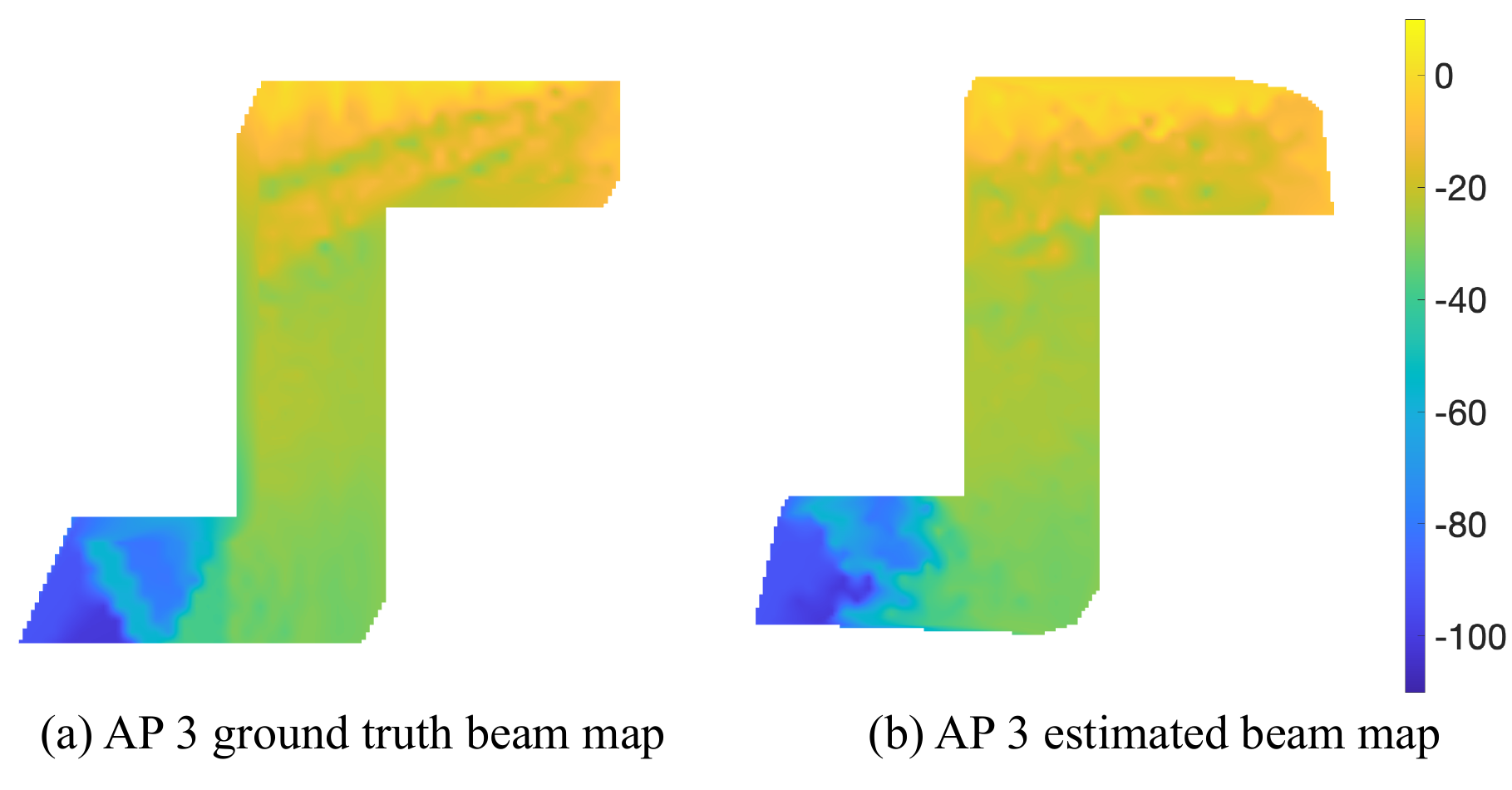}\vspace{-0.2in}
\caption{(a) The ground truth beam map in dB (relative value) and (b) the estimated
map. \label{fig:CSIPredition}}
\vspace{-0.22in}
\end{figure}

Figure \ref{fig:CSIPredition} shows the beam map in dB for a particular
beam. The proposed method is able to accurately reconstruct the beam
map. We define the relative error as $E_{\mathrm{map}}=\frac{1}{NQ}\sum|\frac{e_{t,q}-\hat{e}_{t,q}}{e_{t,q}}|,$
where $\hat{e}_{t,q}$ denotes the estimated beam energy at the ground
truth location, based on the beam energy map constructed from the
estimated locations and their corresponding beam energy values. The
average relative reconstruction error is 3.27 \%.

\vspace{-0.09in}

\section{Conclusion}

\label{sec:Conclusion}

This paper presented a blind radio map construction framework that
infers user trajectories from indoor MIMO-OFDM channel measurements
without requiring explicit location information. Theoretically, a
quasi-specular environment model was developed, and a spatial continuity
theorem under NLOS conditions was established, leading to a CSI-distance
metric proportional to the corresponding physical distance. For rectilinear
trajectories under Poisson-distributed AP deployments, it was shown
that the CRLB of localization error vanishes asymptotically, even
with poor angular resolution---demonstrating the theoretical feasibility
of blind localization. Building on these results, a spatially regularized
Bayesian inference framework was formulated to jointly estimate channel
features, LOS/NLOS conditions and user trajectories. Experiments on
a ray-tracing dataset validated the practical effectiveness of the
approach, achieving an average localization error of 0.68 m and abeam
map reconstruction error of 3.3\%. 


\appendices{}

\section{Proof of Theorem \ref{thm:Spatial-Continuity}}

\label{sec:Proof-of-Spatial-conti}

From \eqref{eq:channel-model-single-antenna}, we have
\begin{align}
	\mathbb{E}\{h_{1}^{(m)}h_{2}^{(m)*}\} & =\mathbb{E}\Big\{\sum_{l=1}^{L}\kappa_{1}^{(l)}\kappa_{2}^{(l)*}e^{-jm\frac{2\pi}{M}B(\tau_{1}^{(l)}-\tau_{2}^{(l)})}\nonumber \\
	& \qquad+\sum_{l=1}^{L}\sum_{l'\neq l}\kappa_{1}^{(l)}\kappa_{2}^{(l')*}e^{-jm\frac{2\pi}{M}B(\tau_{1}^{(l)}-\tau_{2}^{(l)})}\Big\}\nonumber \\
	& =C(d)\cdot L\cdot\mathbb{E}\Big\{ e^{-jm\frac{2\pi}{M}B(\tau_{1}^{(l)}-\tau_{2}^{(l)})}\Big\}\label{eq:Expetation_h1_h2-eq1}
\end{align}
where the second equality is due to the zero-mean \ac{iid} assumption
of the path gains, \emph{i.e.}, $\mathbb{E}\{\kappa_{1}^{(l)}\kappa_{2}^{(l')*}\}=0$
for $l\neq l'$, and the correlation assumption $\mathbb{E}\{\kappa_{1}^{(l)}\kappa_{2}^{(l)*}\}=C(d)$
for the small movement.

For the $l$th path that comes from the mirror image $A_{l}\in\mathcal{A}^{1}=\mathcal{A}^{2}$,
denote the movement direction $\mathbf{p}_{2}-\mathbf{p}_{1}$ relative
to the $l$th arriving at $\mathbf{p}_{1}$ as $\theta'_{l}$. It
is natural to consider that $\theta'_{l}$ is uniformly distributed
in $[-\pi,\pi)$. As a result, for the $l$th path, the delays satisfy
$\tau_{2}^{(l)}-\tau_{1}^{(l)}=\frac{d}{c}\cos(\theta'_{l})$. Denoting
$\omega=B\frac{d}{c}$ for simplification, \eqref{eq:Expetation_h1_h2-eq1}
becomes
\begin{align}
	\mathbb{E}\{h_{1}^{(m)}h_{2}^{(m)*}\} & =C(d)\cdot L\cdot\mathbb{E}\Big\{ e^{-jm\frac{2\pi}{M}\omega\cos(\theta_{l}')}\Big\}\nonumber \\
	& =C(d)L\int_{-\pi}^{\pi}e^{-jm\frac{2\pi}{M}\omega\cos(\theta_{l}')}\frac{1}{2\pi}d\theta_{l}'\nonumber \\
	& =C(d)L\cdot J_{0}(m\frac{2\pi}{M}\omega)\label{eq:Expectation_h1_h2-eq2}
\end{align}
where $J_{0}(x)$ is the zero-th order Bessel function of the first
kind {[}Formula 9.1.21, \cite{AbrMil:B65}{]}.

Substituting \eqref{eq:Expectation_h1_h2-eq2} to \eqref{eq:Ru} for
computing the \ac{idft} of \eqref{eq:Expectation_h1_h2-eq2} yields
\begin{align}
	R(u) & =C(d)L\cdot\frac{1}{M}\sum_{m=0}^{M-1}J_{0}(m\frac{2\pi}{M}\omega)e^{jm\frac{2\pi}{M}u}\label{eq:Ru-Riemann-sum}\\
	& \approx C(d)L\cdot\int_{0}^{1}J_{0}(2\pi f\omega)e^{j2\pi fu}df\label{eq:Ru_Riemann-integral}
\end{align}
for large $M$, where \eqref{eq:Ru-Riemann-sum} becomes a Riemann
sum for ``frequency bin'' $f_{m}=(\frac{m}{M},\frac{m+1}{M})$ with
interval $\frac{1}{M}$, and hence, asymptotically for $M\to\infty$,
$f_{m}$ becomes $f\in(0,1)$, $\frac{1}{M}$ becomes $df$, and the
sum is approximated by the integral.

The integral can be converted into an inverse continuous-time Fourier
transform for a function $J_{0}(2\pi f\omega)\text{rect}(f)$, where
$\text{rect}(f)$ is a rectangle function that acts as a window with
$\text{rect}(f)=1$ for $f\in[0,1]$ and $\text{rect}(f)=0$ otherwise.
If we denote $R_{J}(u)=\mathcal{F}^{-1}\{J_{0}(2\pi f\omega)\}$ as
the inverse Fourier transform of $J_{0}(2\pi f\omega)$ and $R_{r}(u)=\mathcal{F}^{-1}\{\text{rect}(f)\}$,
then according to the convolution property of Fourier transform, we
have $\mathcal{F}\{R_{J}(u)*R_{r}(u)\}=J_{0}(2\pi f\omega)\text{rect}(f)$,
where $\mathcal{F}\{\cdot\}$ denotes the Fourier transform and the
convolution is defined as
\[
R_{J}(u)*R_{r}(u)=\int_{-\infty}^{\infty}R_{J}(u-t)R_{r}(t)dt.
\]
It follows that $R(u)\approx C(d)L\cdot R_{J}(u)*R_{r}(u)$.

The inverse Fourier transform can be computed as
\begin{align*}
	R_{J}(u) & =\int_{-\infty}^{\infty}J_{0}(2\pi f\omega)e^{j2\pi fu}df\\
	& =2\int_{0}^{\infty}J_{0}(2\pi f\omega)\cos(2\pi fu)df\\
	& =\begin{cases}
		\begin{array}{l}
			\frac{1}{\pi\sqrt{\omega^{2}-u^{2}}}\\
			0
		\end{array} & \begin{array}{l}
			|u|<\omega\\
			|u|\geq\omega
	\end{array}\end{cases}
\end{align*}
where the second equality is due to the fact that the Bessel function
$J_{0}(2\pi f\omega)$ is real and even, and the third equality follows
{[}Formula 6.671, \cite{GraIzr:B14}{]}. In addition,
\begin{align*}
	R_{r}(u) & =\int_{0}^{1}e^{j2\pi fu}df=e^{j\pi u}\frac{\sin(\pi u)}{\pi u}=e^{j\pi u}\text{sinc}(u).
\end{align*}

Therefore, we have
\[
R(u)\approx C(d)L\bigg(\frac{\mathbb{I}\{|u|<\omega\}}{\pi\sqrt{\omega^{2}-u^{2}}}\bigg)*\bigg(e^{j\pi u}\text{sinc}(u)\bigg).
\]

It is observed that $R_{J}(u)=\frac{\mathbb{I}\{|u|<\omega\}}{\pi\sqrt{\omega^{2}-u^{2}}}$
is a U-shape function that has its peak at $u\to\omega_{-}$ where
the $R_{J}(u)$ approaches $\infty$. In addition, $\text{sinc}(u)$
is a pulse with ripples and peaked at $u=0$. Thus, the convolution
roughly shifts the pulse $\text{sinc}(u)$ to $u\approx\omega$, and
thus, $|R(u)|$ is peaked at $u\approx\omega=\frac{B}{c}d$, which
can be easily verified by numerical plots.

\section{Proof of Proposition \ref{prop:LB-F-x}}

\label{sec:Proof-of-limited-x}

From (\ref{eq:F-Txtil}), the FIM $\mathbf{F}_{T,x}$, as the upper
diagonal block in $\mathbf{F}_{T,\psi}$, can be expressed as
\begin{align}
	\mathbf{F}_{T,x} & =\sum_{t,q}\frac{1}{\sigma_{\theta}^{2}d_{t,q}^{4}(\mathbf{x},\mathbf{v})}\Big(\lVert\bm{l}_{q}(\mathbf{x})+t\mathbf{v}\rVert^{2}\mathbf{I}\label{eq:F-Tv}\\
	& \qquad-(\bm{l}_{q}(\mathbf{x})+t\mathbf{v})(\bm{l}_{q}(\mathbf{x})+t\mathbf{v})^{\text{T}}\Big)\nonumber 
\end{align}

In the following text, we simplify the notation by writing $\bm{l}_{q}(\mathbf{x})$
as $\bm{l}_{q}$, $d_{t,q}(\mathbf{x},\mathbf{v})$ as $d_{t,q}$,
and $\phi(\mathbf{x}+t\mathbf{v},\mathbf{o}_{q})$ as $\phi_{t,q}$.
\begin{lem}
	\label{prop:CRLB-theta} For a narrow‑band, far‑field source illuminating
	an $N_{\mathrm{t}}$-element uniform linear array, the single‑snapshot
	CRLB for the angle of departure $\theta$ is
	\[
	\mathrm{CRLB}(\hat{\theta})=\frac{d^{2}}{G_{1}N_{\mathrm{t}}(N_{\mathrm{t}}^{2}-1)\cos^{2}\phi}
	\]
	where $G_{1}$ is a constant that depends on the antenna configuration
	with $G_{1}/d^{2}\propto\mathrm{SNR}$, $d$ denotes the distance
	between the transmitter and the receiver, and $\phi$ is the angle
	between the transmitter and the receiver.
\end{lem}

For each $q$, we have $\cos^{2}\phi_{t,q}\leq1$ for all $t$ with
equality achieved when $\phi_{t,q}=0$. Denote $d_{\mathrm{min}}=\min_{q}\{d_{\mathrm{min},q}\}$.
Thus , using Lemma \ref{prop:CRLB-theta}, the angle variance between
the \ac{ap} and the user location is bounded by
\[
\frac{d_{t,q}^{2}\sigma_{\mathrm{n}}^{2}}{G_{1}N_{t}(N_{t}^{2}-1)\cos^{2}\phi_{t,q}}\geq\frac{d_{\mathrm{min}}^{2}\sigma_{\mathrm{n}}^{2}}{G_{1}N_{\mathrm{t}}(N_{\mathrm{t}}^{2}-1)}
\]
for any $t,q$.

Thus, we have
\[
\sigma_{\theta}^{2}\geq\frac{d_{\mathrm{min}}^{2}\sigma_{\mathrm{n}}^{2}}{G_{1}N_{\mathrm{t}}(N_{\mathrm{t}}^{2}-1)}
\]
for any $t,q$.

Thus, $\mathbf{F}_{T,x}$ can be written as
\begin{align}
	\mathbf{F}_{T,x} & \preceq C_{0}\mathbf{A}_{T,x}\label{eq:A-Tx}
\end{align}
where
\begin{align*}
	\mathbf{A}_{T,x} & =\sum_{q=1}^{Q}[s_{T,q}^{(0)}(\bm{l}_{q}^{\mathrm{T}}\bm{l}_{q}\mathbf{I}-\bm{l}_{q}\bm{l}_{q}^{\mathrm{T}})+s_{T,q}^{(1)}(2\mathbf{v}^{\mathrm{T}}\bm{l}_{q}\mathbf{I}\\
	& \qquad-\bm{l}_{q}\mathbf{v}^{\mathrm{T}}-\mathbf{v}\bm{l}_{q}^{\mathrm{T}})+s_{T,q}^{(2)}(\|\mathbf{v}\|^{2}\mathbf{I}-\mathbf{v}\mathbf{v}^{\mathrm{T}})]
\end{align*}
and $C_{0}=\frac{N_{\mathrm{t}}(N_{\mathrm{t}}^{2}-1)\sigma_{n}^{2}}{G_{1}d_{\mathrm{min}}^{2}}$,
$s_{T,q}^{(n)}=\sum_{t=1}^{T}\frac{t^{n}}{d_{t,q}^{4}(\mathbf{x},\mathbf{v})}$.
\begin{lem}
	\label{lem:PSD}Assume that the trajectory $\mathbf{x}_{t}$ does
	not pass any of the \ac{ap} location $\mathbf{o}_{q}$. Then, $\mathbf{A}_{T,x}\prec\mathbf{A}_{T+1,x}$
	if at least two vectors in $\{\bm{l}_{1},\bm{l}_{2},\dots,\bm{l}_{Q},\mathbf{v}\}$
	are linear independent.
\end{lem}
\begin{proof}
	
	Denote $\mathbf{d}_{t,q}\triangleq\bm{l}_{q}(\mathbf{x})+t\mathbf{v}=[d_{t,q,1},d_{t,q,2}]^{\mathrm{T}}$
	as the direction from the $q$th \ac{ap} to the user position at
	time slot $t$. Assume that the trajectory $\mathbf{x}_{t}$ does
	not pass any of the \ac{ap} location $\mathbf{o}_{q}$. Thus, $\bm{l}_{q}=\mathbf{x}-\mathbf{o}_{q}\neq\bm{0}$
	for all $q$, $d_{t,q}>0$ for all $t,q$.
	
	The matrix $\mathbf{A}_{T,x}$ can be expressed as
	\begin{align*}
		\mathbf{A}_{T,x} & =\sum_{q=1}^{Q}\sum_{t=1}^{T}\frac{1}{\|\mathbf{d}_{t,q}\|_{2}^{4}}\Big(\|\mathbf{d}_{t,q}\|_{2}^{2}\mathbf{I}-\mathbf{d}_{t,q}\mathbf{d}_{t,q}^{\text{T}}\Big).
	\end{align*}
	
	The incremental matrix is defined as: 
	\begin{align*}
		\Delta\mathbf{A}_{T} & =\mathbf{A}_{T+1,x}-\mathbf{A}_{T,x}=\sum_{q=1}^{Q}\mathbf{P}_{q}
	\end{align*}
	where $\mathbf{d}_{T+1,q}=\bm{l}_{q}+(T+1)\mathbf{v}$ and 
	\[
	\mathbf{P}_{q}=\frac{1}{\|\mathbf{d}_{T+1,q}\|^{2}}\Big(\mathbf{I}-\frac{\mathbf{d}_{T+1,q}\mathbf{d}_{T+1,q}^{\top}}{\|\mathbf{d}_{T+1,q}\|^{2}}\Big).
	\]
	
	The matrix $\mathbf{P}_{q}$ is P.S.D. because of the Cauchy-Schwarz
	inequality. For any non-zero vector $\bm{u}\in\mathbb{R}^{2}$: 
	\begin{equation}
		\bm{u}^{\top}\mathbf{P}_{q}\bm{u}=\frac{1}{\|\mathbf{d}_{T+1,q}\|^{2}}\Big(\|\bm{u}\|^{2}-\frac{(\bm{u}^{\top}\mathbf{d}_{T+1,q})^{2}}{\|\mathbf{d}_{T+1,q}\|^{2}}\Big)\geq0.
	\end{equation}
	Equality holds if and only if $\bm{u}\parallel\mathbf{d}_{T+1,q}$,
	where $\parallel$ represents parallelism. Thus, $\Delta\mathbf{A}_{T}$
	is P.S.D..
	
	Therefore, $\mathbf{A}_{T+1,x}=\mathbf{A}_{T,x}$ holds if and only
	if $\mathbf{u}\parallel\mathbf{d}_{T+1,q}$ for all $q$. Recall $\mathbf{d}_{T+1,q}=\bm{l}_{q}+(T+1)\mathbf{v}$,
	$\mathbf{v}\neq0$, and $\frac{1}{d_{t,q}^{4}}>0$ for all $q$. If
	$\mathbf{u}\parallel\mathbf{d}_{T+1,q}$ for all $q$, then $\bm{l}_{q}=k_{q}\mathbf{u}-(T+1)\mathbf{v}$
	for any $q$, where all vectors $\bm{l}_{q}$ and $\mathbf{v}$ are
	linear combinations of $\mathbf{u}$. This implies that the set $\{\bm{l}_{1},\bm{l}_{2},\dots,\bm{l}_{Q},\mathbf{v}\}$
	spans a subspace of rank at most 1, as all vectors are collinear with
	$\mathbf{u}$. Thus, if at least two vectors in $\{\bm{l}_{1},\bm{l}_{2},\dots,\bm{l}_{Q},\mathbf{v}\}$
	are linearly independent, there exists some $q$ such that $\mathbf{u}\nparallel\mathbf{d}_{T+1,q}$.
	This ensures that $\mathbf{A}_{T+1,x}\neq\mathbf{A}_{T,x}$. Consequently,
	we have shown that $\Delta\mathbf{A}_{T}$ is positive definite when
	at least two vectors from $\{\bm{l}_{1},\bm{l}_{2},\dots,\bm{l}_{Q},\mathbf{v}\}$
	are linearly independent.
	
\end{proof}

Lemma \ref{lem:PSD} proves that $\mathbf{A}_{T,x}\prec\mathbf{A}_{T+1,x}$
if at least two vectors in $\{\bm{l}_{1},\bm{l}_{2},\dots,\bm{l}_{Q},\mathbf{v}\}$
are linear independent. Therefore, $\mathrm{tr}\{\mathbf{F}_{T,x}^{-1}\}\geq\bar{\Delta}_{T,x}\triangleq\mathrm{tr}\{(C_{0}\mathbf{A}_{T,x})^{-1}\}$.

Similar to , we have
\begin{lem}
	(Lemma 8 in \cite{ZheChe:J25})\label{lem:A-Tx-lam}Suppose $d_{\min,q}>0$.
	The sequence $s_{T,q}^{(n)}$ is bounded for $n<3$ and divergent
	as $s_{T,q}^{(n)}\rightarrow\infty$ as $T\rightarrow\infty$ for
	$n\geq3$. In addition, $s_{T,q}^{(n+1)}/s_{T,q}^{(n)}\rightarrow\infty$
	as $T\rightarrow\infty$ for $n>3$.
\end{lem}
Using Lemma \ref{lem:A-Tx-lam}, since $s_{T,q}^{(n)}$ are bounded
for $n<3$ and $Q$ is finite, we have $\mathbf{A}_{T,x}$ bounded.
Thus, $\bar{\Delta}_{T,x}$ converges to a strictly positive number
as $T\rightarrow\infty$.

\section{Proof of Proposition \ref{prop:LB-F-v}}

\label{sec:Proof-of-limited-v}

From (\ref{eq:F-Txtil}), the FIM $\mathbf{F}_{T,v}$, as the lower
diagonal block in $\mathbf{F}_{T,\psi}$ can be expressed as
\begin{align*}
	\mathbf{F}_{T,v} & =\sum_{t,q}\frac{t^{2}}{\sigma_{\theta}^{2}d_{t,q}^{4}(\mathbf{x},\mathbf{v})}\Big(\lVert\bm{l}_{q}(\mathbf{x})+t\mathbf{v}\rVert^{2}\mathbf{I}\\
	& \qquad-(\bm{l}_{q}(\mathbf{x})+t\mathbf{v})(\bm{l}_{q}(\mathbf{x})+t\mathbf{v})^{\text{T}}\Big)
\end{align*}

Similar to (\ref{eq:A-Tx}), we have,
\begin{align}
	\mathbf{F}_{T,v} & \preceq C_{0}\mathbf{A}_{T,v}\label{eq:A-Tv}
\end{align}
where
\begin{align*}
	\mathbf{A}_{T,v} & =\sum_{q=1}^{Q}\Big[s_{T,q}^{(2)}(\bm{l}_{q}^{\mathrm{T}}\bm{l}_{q}\mathbf{I}-\bm{l}_{q}\bm{l}_{q}^{\mathrm{T}})+\sum_{q=1}^{Q}s_{T,q}^{(3)}(2\mathbf{v}^{\mathrm{T}}\bm{l}_{q}\mathbf{I}\\
	& \qquad-\bm{l}_{q}\mathbf{v}^{\mathrm{T}}-\mathbf{v}\bm{l}_{q}^{\mathrm{T}})+\sum_{q=1}^{Q}s_{T,q}^{(4)}(\|\mathbf{v}\|^{2}\mathbf{I}-\mathbf{v}\mathbf{v}^{\mathrm{T}})\Big]
\end{align*}

\begin{lem}
	\label{lem:A-Tv-lam}The eigenvalues of $\mathbf{A}_{T,v}$ satisfies
	\[
	\lambda_{\mathrm{min}}(\mathbf{A}_{T,v})\rightarrow\sum_{q=1}^{Q}s_{T,q}^{(2)}\|\mathbf{P}_{v}^{\bot}\bm{l}_{q}\|^{2}
	\]
	and $\lambda_{\mathrm{max}}(\mathbf{A}_{T,v})\rightarrow\sum_{q=1}^{Q}s_{T,q}^{(4)}\|\mathbf{v}\|^{2}$
	as $T\rightarrow\infty$, where $\mathbf{P}_{v}^{\bot}=\mathbf{I}-\mathbf{vv^{T}/\|v}\|^{2}$.
\end{lem}
\begin{proof}
	
	Since the term $\sum_{q=1}^{Q}s_{t,q}^{(4)}\mathbf{v}\mathbf{v}^{T}$
	dominates $\mathbf{A}_{T,v}$ for a sufficiently large $T$, for a
	sufficiently large $T$, the larger eigenvalue satisfies
	\begin{align}
		\lambda_{\mathrm{max}}(\mathbf{A}_{T,v}) & =\underset{\|\mathbf{u}\|=1}{\mathrm{max}}\:\mathbf{u}^{\mathrm{T}}\mathbf{A}_{T,v}\mathbf{u}\nonumber \\
		& =\underset{\|\mathbf{u}\|=1}{\mathrm{max}}\:\mathbf{u}^{\mathrm{T}}\Big[\sum_{q=1}^{Q}s_{T,q}^{(2)}(\bm{l}_{q}^{\mathrm{T}}\bm{l}_{q}\mathbf{I}-\bm{l}_{q}\bm{l}_{q}^{\mathrm{T}})\nonumber \\
		& \qquad+\sum_{q=1}^{Q}s_{T,q}^{(3)}(2\mathbf{v}^{\mathrm{T}}\bm{l}_{q}\mathbf{I}-\bm{l}_{q}\mathbf{v}^{\mathrm{T}}-\mathbf{v}\bm{l}_{q}^{\mathrm{T}})\nonumber \\
		& \qquad+\sum_{q=1}^{Q}s_{T,q}^{(4)}(\|\mathbf{v}\|^{2}\mathbf{I}-\mathbf{v}\mathbf{v}^{\mathrm{T}})\Big]\mathbf{u}\nonumber \\
		& \approx\underset{\|\mathbf{u}\|=1}{\mathrm{max}}\:\sum_{q=1}^{Q}s_{T,q}^{(4)}\cdot\mathbf{u}^{\mathrm{T}}(\|\mathbf{v}\|^{2}\mathbf{I}-\mathbf{v}\mathbf{v}^{\mathrm{T}})\mathbf{u}\label{eq:prob-uv}
	\end{align}
	where when $\mathbf{u}$ is orthogonal to $\mathbf{v}$ (i.e., $\mathbf{u}^{\mathrm{T}}\mathbf{v}=0$),
	the expression $\|\mathbf{v}\|^{2}-(\mathbf{u}^{\mathrm{T}}\mathbf{v})^{2}=\|\mathbf{v}\|^{2}$
	attains its maximum value, that is $\mathbf{u}=\mathbf{v}_{\perp}/\|\mathbf{v}_{\perp}\|$,
	we have $\mathbf{v}_{\perp}^{\mathrm{T}}\mathbf{v}_{\perp}=1$, $\mathbf{v}^{\mathrm{T}}\mathbf{v}_{\perp}=0$,
	$\mathbf{v}_{\perp}^{\mathrm{T}}\mathbf{v}=0$ and $\lambda_{\mathrm{max}}(\mathbf{A}_{T,v})\rightarrow\sum_{q=1}^{Q}s_{T,q}^{(4)}\|\mathbf{v}\|^{2}$.
	As a result, asymptotically, the larger eigenvector is $\mathbf{v}_{\perp}/\|\mathbf{v}_{\perp}\|\in\mathbb{R}^{2}$,
	and hence, the smaller eigenvector is denoted as $\tilde{\mathbf{u}}=\frac{\mathbf{v}}{\|\mathbf{v}\|_{2}}$.
	
	Consequentially, we have
	\begin{align*}
		& \lambda_{\mathrm{min}}(\mathbf{A}_{T,v})\\
		& =\tilde{\mathbf{u}}^{\mathrm{T}}\Big[\sum_{q=1}^{Q}s_{T,q}^{(2)}(\bm{l}_{q}^{\mathrm{T}}\bm{l}_{q}\mathbf{I}-\bm{l}_{q}\bm{l}_{q}^{\mathrm{T}})+\sum_{q=1}^{Q}s_{T,q}^{(3)}(2\mathbf{v}^{\mathrm{T}}\bm{l}_{q}\mathbf{I}\\
		& \qquad-\bm{l}_{q}\mathbf{v}^{\mathrm{T}}-\mathbf{v}\bm{l}_{q}^{\mathrm{T}})+\sum_{q=1}^{Q}s_{T,q}^{(4)}(\|\mathbf{v}\|^{2}\mathbf{I}-\mathbf{v}\mathbf{v}^{\mathrm{T}})\Big]\tilde{\mathbf{u}}\\
		& =\sum_{q=1}^{Q}s_{T,q}^{(2)}\|\mathbf{P}_{v}^{\bot}\bm{l}_{q}\|^{2}.
	\end{align*}
	where $\mathbf{P}_{v}^{\bot}=\mathbf{I}-\mathbf{vv^{T}/\|v}\|^{2}$
	is orthogonal projector, and $\mathbf{P}_{v}^{\bot}\bm{l}_{q}$ is
	to project the vector $\bm{l}_{q}$ onto the null space spanned by
	$\mathbf{v}_{\bot}$ of $\mathbf{v}$.
	
\end{proof}

From $\mathbf{F}_{T,v}\preceq C_{0}\mathbf{A}_{T,v}$, since both
$\mathbf{F}_{T,v}$ and $\mathbf{A}_{T,v}$ are P.S.D., we have
\begin{equation}
	\lambda_{\text{min}}(\mathbf{F}_{T,v})\leq C_{0}\lambda_{\text{min}}(\mathbf{A}_{T,v}),\lambda_{\text{max}}(\mathbf{F}_{T,v})\leq C_{0}\lambda_{\text{max}}(\mathbf{A}_{T,v}).\label{eq:lam-min-ineq}
\end{equation}

Denoting the \ac{evd} of $\mathbf{F}_{T,\mathbf{v}}$ as $\mathbf{F}_{T,\mathbf{v}}=\mathbf{u}_{T,v}\bm{\varLambda}_{T,v}\mathbf{u}_{T,v}^{-1}$,
where
\[
\bm{\varLambda}_{T,v}=\left[\begin{array}{cc}
	\lambda_{\mathrm{max}}(\mathbf{A}_{T,\mathbf{v}}) & 0\\
	0 & \lambda_{\mathrm{min}}(\mathbf{A}_{T,\mathbf{v}})
\end{array}\right],
\]
we have
\begin{align}
	\mathrm{tr}\{\mathbf{F}_{T,v}^{-1}\} & =\mathrm{tr}\{(\mathbf{u}_{T,v}\bm{\varLambda}_{T,v}\mathbf{u}_{T,v}^{-1})^{-1}\}=\mathrm{tr}\{\mathbf{u}_{T,v}\bm{\varLambda}_{T,v}^{-1}\mathbf{u}_{T,v}^{-1}\}\nonumber \\
	& =\lambda_{\mathrm{max}}^{-1}(\mathbf{F}_{T,v})+\lambda_{\mathrm{min}}^{-1}(\mathbf{F}_{T,v})\nonumber \\
	& \geq\frac{1}{C_{0}}\lambda_{\mathrm{max}}^{-1}(\mathbf{A}_{T,v})+\frac{1}{C_{0}}\lambda_{\mathrm{min}}^{-1}(\mathbf{A}_{T,v})\label{eq:tfF-ineq1}\\
	& \geq\frac{1}{C_{0}}\lambda_{\mathrm{min}}^{-1}(\mathbf{A}_{T,v})\triangleq\bar{\Delta}_{T,v}\label{eq:tfF-ineq2}
\end{align}
where (\ref{eq:tfF-ineq1}) is due to (\ref{eq:lam-min-ineq}) with
equality achieved when $d_{t,q}^{2}=d_{\mathrm{min}}^{2}$, $\cos^{2}\phi_{t,q}=1$
and (\ref{eq:tfF-ineq2}) is due to the fact that $C_{0}\lambda_{\mathrm{max}}(\mathbf{A}_{T,\mathbf{v}})>0$
and equality can be asymptotically achieved at large $T$ as $\lambda_{\mathrm{\max}}^{-1}(\mathbf{A}_{T,v})\rightarrow1/(\sum_{q=1}^{Q}s_{T,q}^{(4)}\|\mathbf{v}\|^{2})$
which converges to zero.

Using Lemma \ref{lem:A-Tv-lam}, as $T\rightarrow\infty$, we have
\begin{align*}
	\bar{\Delta}_{T,v} & \rightarrow C_{v}=\Bigg(C_{0}\sum_{q=1}^{Q}s_{T,q}^{(2)}\|\mathbf{P}_{v}^{\bot}\bm{l}_{q}\|^{2}\Bigg)^{-1},
\end{align*}
which is strictly positive, where $\mathbf{P}_{v}^{\bot}=\mathbf{I}-\mathbf{vv^{T}/\|v}\|^{2}$
is orthogonal projector, and $\mathbf{P}_{v}^{\bot}\bm{l}_{q}$ is
to project the vector $\bm{l}_{q}$ onto the null space spanned by
$\mathbf{v}_{\bot}$ of $\mathbf{v}$. $s_{T,q}^{(2)}$ is bounded
as stated in Lemma \ref{lem:A-Tx-lam}. Suppose $\rho>0$ is sufficiently
small such that $d_{t,q}>\rho t$ for all $t\geq1$, we have
\begin{align*}
	s_{\infty,q}^{(2)} & =\lim_{T\to\infty}\sum_{t=1}^{T}\frac{t^{2}}{d_{t,q}^{4}}<\lim_{T\to\infty}\sum_{t=1}^{T}\frac{t^{2}}{(\rho t)^{4}}\\
	& =\frac{1}{\rho^{4}}\lim_{T\to\infty}\sum_{t=1}^{T}\frac{1}{t^{2}}\approx\frac{\pi^{2}}{6\rho^{4}}.
\end{align*}
Thus, the element $s_{\infty,q}^{(2)}$ is upper bounded by $\frac{\pi^{2}}{6\rho^{4}}$.

\section{Proof of Theorem \ref{thm:LB-F-xv-un}}

\label{sec:Proof-of-unlimited-x}

Consider the \ac{crlb} $B(\mathbf{x})$. Denote $\mathcal{Q}_{t}=\{q|d_{t,q}\leq R\}$
as the set of \acpl{ap} that are within a range of $R$ from the
mobile user at time slot $t$. Based on the FIM $\mathbf{F}_{T,\psi}$
in (\ref{eq:F-Tv}), we have
\begin{align}
	\mathbf{F}_{T,x} & =\sum_{t=1}^{T}\mathbb{E}\bigg\{\sum_{q\in\mathcal{Q}_{t}}\frac{1}{\sigma_{\theta}^{2}d_{t,q}^{4}(\mathbf{x},\mathbf{v})}(\lVert\bm{l}_{q}(\mathbf{x})+t\mathbf{v}\rVert^{2}\mathbf{I}\label{eq:ineq-F-tx}\\
	& \qquad-(\bm{l}_{q}(\mathbf{x})+t\mathbf{v})(\bm{l}_{q}(\mathbf{x})+t\mathbf{v})^{\text{T}})\Bigg\}\nonumber 
\end{align}
For each term in the sum, the geometry-dependent Fisher information
per sample is less than or equal to that from the uniform minimum-variance
case. Thus, using Lemma \ref{prop:CRLB-theta}, we have
\begin{align*}
	\mathbf{F}_{T,x} & \succeq\frac{G_{1}N_{\mathrm{t}}(N_{\mathrm{t}}^{2}-1)}{\sigma_{\mathrm{n}}^{2}}\sum_{t=1}^{T}\mathbb{E}\bigg\{\sum_{q\in\mathcal{Q}_{t}}\frac{d_{t,q,1}^{2}}{d_{t,q}^{8}}(\lVert\bm{l}_{q}(\mathbf{x})+t\mathbf{v}\rVert^{2}\mathbf{I}\\
	& \qquad-(\bm{l}_{q}(\mathbf{x})+t\mathbf{v})(\bm{l}_{q}(\mathbf{x})+t\mathbf{v})^{\text{T}})\Bigg\}\\
	& =\tilde{C}_{0}\tilde{\mathbf{A}}_{T,x}
\end{align*}
where $\tilde{C}_{0}=G_{1}N_{\mathrm{t}}(N_{\mathrm{t}}^{2}-1)/\sigma_{\mathrm{n}}^{2}$
and
\begin{align}
	\tilde{\mathbf{A}}_{T,x} & =\sum_{t=1}^{T}\mathbb{E}\Bigg\{\sum_{q\in\mathcal{Q}_{t}}d_{t,q,1}^{2}\frac{\bm{l}_{q}^{\mathrm{T}}\bm{l}_{q}\mathbf{I}-\bm{l}_{q}\bm{l}_{q}^{\mathrm{T}}}{d_{t,q}^{8}}\Bigg\}\nonumber \\
	& \quad\quad+\sum_{t=1}^{T}t\mathbb{E}\Bigg\{\sum_{q\in\mathcal{Q}_{t}}d_{t,q,1}^{2}\frac{2\mathbf{v}^{\mathrm{T}}\bm{l}_{q}\mathbf{I}-\bm{l}_{q}\mathbf{v}^{\mathrm{T}}-\mathbf{v}\bm{l}_{q}^{\mathrm{T}}}{d_{t,q}^{8}}\Bigg\}\nonumber \\
	& \quad\quad+\sum_{t=1}^{T}t^{2}\mathbb{E}\Bigg\{\sum_{q\in\mathcal{Q}_{t}}d_{t,q,1}^{2}\frac{\|\mathbf{v}\|^{2}\mathbf{I}-\mathbf{v}\mathbf{v}^{\mathrm{T}}}{d_{t,q}^{8}}\Bigg\}.\label{eq:A-tilde-T-x}
\end{align}

Since $\mathbf{F}_{T,x}$ and $\tilde{\mathbf{A}}_{T,x}$ are $2\times2$
symmetric and positive semi-definite, their eigenvalues are real and
non-negative. From $\mathbf{F}_{T,x}\succeq\tilde{C}_{0}\tilde{\mathbf{A}}_{T,x}$,
we have
\[
\lambda_{\text{min}}(\mathbf{F}_{T,x})\geq\tilde{C}_{0}\lambda_{\text{min}}(\tilde{\mathbf{A}}_{T,x}),\lambda_{\text{max}}(\mathbf{F}_{T,x})\geq\tilde{C}_{0}\lambda_{\text{max}}(\tilde{\mathbf{A}}_{T,x}).
\]

Since $\mathrm{tr}\{\mathbf{F}_{T,x}^{-1}\}=\lambda_{\mathrm{max}}^{-1}(\mathbf{F}_{T,x})+\lambda_{\mathrm{min}}^{-1}(\mathbf{F}_{T,x})$,
we have
\begin{equation}
	\mathrm{tr}\{\mathbf{F}_{T,x}^{-1}\}\leq2\lambda_{\mathrm{min}}^{-1}(\mathbf{F}_{T,x})\leq2\left(\tilde{C}_{0}\lambda_{\mathrm{min}}(\tilde{\mathbf{A}}_{T,x})\right)^{-1}\triangleq\tilde{\Delta}_{T,x}.\label{eq:tr-F-ineq}
\end{equation}

\begin{lem}
	\label{lem:A-Tx-lam-1}Assume that $d_{t,q}\geq r_{0}$ \textup{for
		all $t$ and $q$}. The eigenvalue of $\tilde{\mathbf{A}}_{T,x}$
	satisfies
	\[
	\frac{1}{T}\lambda_{\mathrm{min}}(\tilde{\mathbf{A}}_{T,x})\rightarrow\frac{1}{8}\pi\kappa(\frac{1}{r_{0}^{2}}-\frac{1}{R^{2}})
	\]
	as $T\to\infty$.
\end{lem}
\begin{proof}The term $\sum_{t=1}^{T}t^{2}\mathbb{E}\{\sum_{q\in\mathcal{Q}_{t}}d_{t,q,2}^{2}(\|\mathbf{v}\|^{2}\mathbf{I}-\mathbf{v}\mathbf{v}^{\mathrm{T}})/d_{t,q}^{8}\}$
	in (\ref{eq:A-tilde-T-x}) dominates $\tilde{\mathbf{A}}_{T,x}$ for
	a sufficiently large $T$, because $t^{2}$ increases quadratically.
	Thus, as $T\rightarrow\infty$, the larger eigenvalue satisfies: 
	\begin{align}
		\frac{1}{T}\lambda_{\mathrm{max}}(\tilde{\mathbf{A}}_{T,x}) & =\frac{1}{T}\underset{\|\tilde{\mathbf{u}}\|=1}{\mathrm{max}}\:\tilde{\mathbf{u}}^{\mathrm{T}}\tilde{\mathbf{A}}_{T,x}\tilde{\mathbf{u}}\nonumber \\
		& \rightarrow\underset{\|\tilde{\mathbf{u}}\|=1}{\mathrm{max}}\:\frac{1}{T}\sum_{t=1}^{T}t^{2}\mathbb{E}\Bigg\{\sum_{q\in\mathcal{Q}_{t}}d_{t,q,1}^{2}\frac{1}{d_{t,q}^{8}}\label{eq:prob-uv-1}\\
		& \qquad\times\tilde{\mathbf{u}}^{\mathrm{T}}(\|\mathbf{v}\|^{2}\mathbf{I}-\mathbf{v}\mathbf{v}^{\mathrm{T}})\tilde{\mathbf{u}}\Bigg\}\nonumber 
	\end{align}
	where the solution to \eqref{eq:prob-uv-1} is $\tilde{\mathbf{u}}=\mathbf{v}_{\perp}/\|\mathbf{v}_{\perp}\|_{2}$,
	which satisfies $\mathbf{v}^{\mathrm{T}}\mathbf{v}_{\perp}=0$.
	
	As a result, asymptotically, the larger eigenvector is $\mathbf{v}_{\perp}/\|\mathbf{v}_{\perp}\|_{2}\in\mathbb{R}^{2}$
	and hence, the smaller eigenvector is $\mathbf{u}=\frac{\mathbf{v}}{\|\mathbf{v}\|_{2}}$,
	. Consequently, from (\ref{eq:A-tilde-T-x}), as $T\to\infty$, we
	have: 
	\begin{align}
		& \frac{1}{T}\lambda_{\mathrm{min}}(\tilde{\mathbf{A}}_{T,x})\nonumber \\
		& \rightarrow\frac{1}{T}\sum_{t=1}^{T}\mathbb{E}\left\{ \sum_{q\in\mathcal{Q}_{t}}\frac{d_{t,q,1}^{2}}{d_{t,q}^{8}}\|\mathbf{P}_{v}^{\bot}\bm{l}_{q}\|^{2}\right\} .\label{eq:E-stq0}
	\end{align}
	
	To compute the expectation in (\ref{eq:E-stq0}), we note that as
	the \acpl{ap} follow a Poisson distribution within a radius of $R$
	from the user location $\mathbf{x}_{t}$, the expected number of the
	\acpl{ap} is $\kappa\pi R^{2}$. In addition, given the number of
	the \acpl{ap}, the \acpl{ap} are independently and uniformly distributed.
	As a result, consider a coordinate system with the initial position
	$\mathbf{x}$ as the origin and the direction $\mathbf{v}$ as the
	$x$-axis as stated in Lemma 10 in \cite{ZheChe:J25}, and then, $\mathbf{P}_{v}^{\bot}\bm{l}_{q}$
	is simply to project the vector $\bm{l}_{q}=\mathbf{x}-\mathbf{o}_{q}$
	onto the $y$-axis. Denote $\bm{l}_{q}=(l_{q,x},l_{q,y})$ and it
	follows that $\mathbf{P}_{v}^{\bot}\bm{l}_{q}=l_{q,y}$ and $d_{t,q,1}^{2}=l_{q,x}^{2}$.
	
	We have
	\begin{align}
		& \mathbb{E}\left\{ \sum_{q\in\mathcal{Q}_{t}}\frac{d_{t,q,1}^{2}}{d_{t,q}^{8}}\|\mathbf{P}_{v}^{\bot}\bm{l}_{q}\|^{2}\right\} \nonumber \\
		& =\mathbb{E}\left\{ \frac{l_{q,y}^{2}l_{q,x}^{2}}{(l_{q,x}^{2}+l_{q,y}^{2})^{4}}\right\} \kappa\pi R^{2}\nonumber \\
		& =\kappa\pi R^{2}\frac{1}{\pi R^{2}}\int_{-R}^{R}\int_{-\sqrt{R^{2}-x^{2}}}^{\sqrt{R^{2}-x^{2}}}\frac{x^{2}y^{2}}{(x^{2}+y^{2})^{4}}\,dy\,dx\label{eq:remove-T}\\
		& =\kappa\frac{\pi}{8}(\frac{1}{r_{0}^{2}}-\frac{1}{R^{2}}).\label{eq:E-pl}
	\end{align}
	where (\ref{eq:remove-T}) is due to the prior condition that $d_{t,q}>r_{0}$
	for all $t$ and $q$.
	
	As a result, from (\ref{eq:E-stq0}) and (\ref{eq:E-pl}), we have
	$\frac{1}{T}\lambda_{\mathrm{min}}(\tilde{\mathbf{A}}_{T,x})\rightarrow\frac{1}{8}\kappa\pi(\frac{1}{r_{0}^{2}}-\frac{1}{R^{2}})$
	as $T\to\infty$.\end{proof}

According Lemma \ref{lem:A-Tx-lam-1} and from (\ref{eq:tr-F-ineq}),
we have
\begin{align*}
	T\tilde{\Delta}_{T,x} & \rightarrow\frac{16}{\kappa\pi(r_{0}^{-2}-R^{-2})G_{1}N_{\mathrm{t}}(N_{\mathrm{t}}^{2}-1)}
\end{align*}
as $T\to\infty$.

Consider the \ac{crlb} $B(\mathbf{v})$. Based on the FIM $\mathbf{F}_{T,v}$
in (\ref{eq:F-Tv}), we have
\begin{align}
	\mathbf{F}_{T,v} & =\sum_{t=1}^{T}\mathbb{E}\Big\{\sum_{q\in\mathcal{Q}_{t}}\frac{t^{2}}{\sigma_{\theta,t,q}^{2}d_{t,q}^{4}(\mathbf{x},\mathbf{v})}[\lVert\bm{l}_{q}(\mathbf{x})+t\mathbf{v}\rVert^{2}\mathbf{I}\label{eq:ineq-Ftv}\\
	& -(\bm{l}_{q}(\mathbf{x})+t\mathbf{v})(\bm{l}_{q}(\mathbf{x})+t\mathbf{v})^{\text{T}}]\Big\}\nonumber 
\end{align}
Using Lemma \ref{prop:CRLB-theta}, we have
\begin{align*}
	\mathbf{F}_{T,x} & \succeq\frac{G_{1}N_{\mathrm{t}}(N_{\mathrm{t}}^{2}-1)}{\sigma_{\mathrm{n}}^{2}}\sum_{t=1}^{T}\mathbb{E}\bigg\{\sum_{q\in\mathcal{Q}_{t}}\frac{d_{t,q,1}^{2}t^{2}}{d_{t,q}^{8}}(\lVert\bm{l}_{q}(\mathbf{x})+t\mathbf{v}\rVert^{2}\mathbf{I}\\
	& \qquad-(\bm{l}_{q}(\mathbf{x})+t\mathbf{v})(\bm{l}_{q}(\mathbf{x})+t\mathbf{v})^{\text{T}})\Bigg\}\\
	& =\tilde{C}_{0}\sum_{t=1}^{T}\mathbb{E}\bigg\{\sum_{q\in\mathcal{Q}_{t}}d_{t,q,1}^{2}t^{2}\Big[\frac{\bm{l}_{q}^{\mathrm{T}}\bm{l}_{q}\mathbf{I}-\bm{l}_{q}\bm{l}_{q}^{\mathrm{T}}}{d_{t,q}^{8}}\\
	& +t\frac{2\mathbf{v}^{\mathrm{T}}\bm{l}_{q}\mathbf{I}-\bm{l}_{q}\mathbf{v}^{\mathrm{T}}-\mathbf{v}\bm{l}_{q}^{\mathrm{T}}}{d_{t,q}^{8}}+t^{2}\frac{\|\mathbf{v}\|^{2}\mathbf{I}-\mathbf{v}\mathbf{v}^{\mathrm{T}}}{d_{t,q}^{8}}\Big]\Bigg\}\\
	& =\tilde{C}_{0}\tilde{\mathbf{A}}_{T,v}
\end{align*}
where
\begin{align}
	\tilde{\mathbf{A}}_{T,v} & =\sum_{t=1}^{T}\mathbb{E}\Bigg\{\sum_{q\in\mathcal{Q}_{t}}d_{t,q,1}^{2}t^{2}\frac{\bm{l}_{q}^{\mathrm{T}}\bm{l}_{q}\mathbf{I}-\bm{l}_{q}\bm{l}_{q}^{\mathrm{T}}}{d_{t,q}^{8}}\Bigg\}\nonumber \\
	& \quad\quad+\sum_{t=1}^{T}t\mathbb{E}\Bigg\{\sum_{q\in\mathcal{Q}_{t}}d_{t,q,1}^{2}t^{2}\frac{2\mathbf{v}^{\mathrm{T}}\bm{l}_{q}\mathbf{I}-\bm{l}_{q}\mathbf{v}^{\mathrm{T}}-\mathbf{v}\bm{l}_{q}^{\mathrm{T}}}{d_{t,q}^{8}}\Bigg\}\nonumber \\
	& \quad\quad+\sum_{t=1}^{T}t^{2}\mathbb{E}\Bigg\{\sum_{q\in\mathcal{Q}_{t}}d_{t,q,1}^{2}t^{2}\frac{\|\mathbf{v}\|^{2}\mathbf{I}-\mathbf{v}\mathbf{v}^{\mathrm{T}}}{d_{t,q}^{8}}\Bigg\}.\label{eq:A-tilde-T-v}
\end{align}

Since $\mathbf{F}_{T,v}$ is symmetric and positive semi-definite,
its eigenvalues are real and non-negative. Similar to (\ref{eq:tr-F-ineq}),
we have
\begin{equation}
	\mathrm{tr}\{\mathbf{F}_{T,v}^{-1}\}\leq2\lambda_{\mathrm{min}}^{-1}(\mathbf{F}_{T,v})\leq2[\tilde{C}_{0}\lambda_{\mathrm{min}}(\tilde{\mathbf{A}}_{T,v})]^{-1}\triangleq\tilde{\Delta}_{T,v}.\label{eq:tr-F-ineq-1}
\end{equation}

\begin{lem}
	\label{lem:A-Tx-lam-1-1}Under the same condition in Lemma \ref{lem:A-Tx-lam-1},
	the eigenvalue of $\tilde{\mathbf{A}}_{T,v}$ satisfies
	\[
	\frac{\lambda_{\mathrm{min}}(\tilde{\mathbf{A}}_{T,x})}{T(T+1)(2T+1)}\rightarrow\frac{\pi}{48}\kappa(\frac{1}{r_{0}^{2}}-\frac{1}{R^{2}})
	\]
	as $T\to\infty$.
\end{lem}
\begin{proof}Similar to the derivation of Lemma \ref{lem:A-Tx-lam-1},
	the asymptotic larger eigenvector of $\tilde{\mathbf{A}}_{T,v}$ is
	$\mathbf{u}=\mathbf{v}_{\perp}/\|\mathbf{v}_{\perp}\|_{2}$, because
	the last term in (\ref{eq:A-tilde-T-v}) dominates when $T$ is large.
	
	As a result, asymptotically, the smaller eigenvector is $\mathbf{v}/\|\mathbf{v}\|_{2}$.
	Consequently, from (\ref{eq:A-tilde-T-v}), as $T\to\infty$, we have
	\begin{align*}
		\frac{\lambda_{\mathrm{min}}(\tilde{\mathbf{A}}_{T,v})}{T(T+1)(2T+1)} & \rightarrow\frac{\kappa\pi}{48}(\frac{1}{r_{0}^{2}}-\frac{1}{R^{2}})
	\end{align*}
	as $T\to\infty$.\end{proof}

According Lemma \ref{lem:A-Tx-lam-1-1} and from (\ref{eq:tr-F-ineq-1}),
we have
\begin{align*}
	& T(T+1)(2T+1)\tilde{\Delta}_{T,v}\rightarrow\frac{96}{\kappa\pi(r_{0}^{-2}-R^{-2})G_{1}N_{\mathrm{t}}(N_{\mathrm{t}}^{2}-1)}
\end{align*}
as $T\to\infty$.

\vspace{-0.09in}

\bibliographystyle{IEEEtran}
\bibliography{my_ref}

@article{StaYam:J23,
  title={Indoor localization with robust global channel charting: A time-distance-based approach},
  author={Stahlke, Maximilian and Yammine, George and Feigl, Tobias and Eskofier, Bjoern M and Mutschler, Christopher},
  journal={IEEE Trans. Mach. Learn. Commun. Netw.},
  volume={1},
  pages={3--17},
  year={2023}
}

@inproceedings{TanPal:C23,
  title={Channel charting in real-world coordinates},
  author={Taner, Sueda and Palhares, Victoria and Studer, Christoph},
  booktitle={Proc. IEEE Global Commun. Conf. (GlobeCom)},
  pages={3940--3946},
  year={2023},
  organization={IEEE}
}

@article{TanPal:J25,
  title={Channel charting in real-world coordinates with distributed {MIMO}},
  author={Taner, Sueda and Palhares, Victoria and Studer, Christoph},
  journal={IEEE Trans. Wireless Commun.},
  year={2025},
  volume={24},
  number={9},
  pages={7286-7300}
}

@article{MonTod:J96,
  title={The expectation-maximization algorithm},
  author={Moon, Todd K},
  journal={IEEE Signal Process. Mag.},
  volume={13},
  number={6},
  pages={47--60},
  year={1996}
}

@ARTICLE{ShiChe:J24,
  author={Shi, Chenghao and Chen, Xieyuanli and Xiao, Junhao and Dai, Bin and Lu, Huimin},
  journal={IEEE Trans. Robot.}, 
  title={Fast and Accurate Deep Loop Closing and Relocalization for Reliable {LiDAR} {SLAM}}, 
  year={2024},
  volume={40},
  number={},
  pages={2620-2640}
}

@article{WanMao:J19,
  title={Joint time-of-arrival estimation for coherent {UWB} ranging in multipath environment with multi-user interference},
  author={Wang, Shangbo and Mao, Guoqiang and Zhang, J Andrew},
  journal={IEEE Trans. Signal Process.},
  volume={67},
  number={14},
  pages={3743--3755},
  year={2019}
}

@ARTICLE{KonQi:J19,
  author={Kong, Lingchen and Qi, Chuanqi and Qi, Hou-Duo},
  journal={IEEE Trans. Signal Process.},
  title={Classical Multidimensional Scaling: A Subspace Perspective, Over-Denoising, and Outlier Detection}, 
  year={2019},
  volume={67},
  number={14},
  pages={3842-3857}
}

@article{RomDan:J22,
  title        = {Radio map estimation: {A} data-driven approach to spectrum cartography},
  author       = {Romero, Daniel and Kim, Seung-Jun},
  journal      = {IEEE Signal Process. Mag.},
  volume       = {39},
  number       = {6},
  pages        = {53--72},
  year         = {2022}
}

@article{ZenChe:J24,
  title={A tutorial on environment-aware communications via channel knowledge map for {6G}},
  author={Zeng, Yong and Chen, Junting and Xu, Jie and Wu, Di and Xu, Xiaoli and Jin, Shi and Gao, Xiqi and Gesbert, David and Cui, Shuguang and Zhang, Rui},
  journal={IEEE Commun. Surv. Tutorials},
volume={26},
  number={3},
  pages={1478--1519},
  year={2024}
}

@article{ZenXuX:J21,
  title={Toward environment-aware {6G} communications via channel knowledge map},
  author={Zeng, Yong and Xu, Xiaoli},
  journal={IEEE Wireless Commun.},
  volume={28},
  number={3},
  pages={84--91},
  year={2021}
}

@article{LiLiao:J23,
  title={Automatic indoor radio map construction and localization via multipath fingerprint extrapolation},
  author={Li, Qiao and Liao, Xuewen and Li, Ang and Valaee, Shahrokh},
  journal={IEEE Trans. Wireless Commun.},
  volume={22},
  number={9},
  pages={5814--5827},
  year={2023}
}

@article{TeqRom:J21,
  title={Deep completion autoencoders for radio map estimation},
  author={Teganya, Yves and Romero, Daniel},
  journal={IEEE Trans. Wireless Commun.},
  volume={21},
  number={3},
  pages={1710--1724},
  year={2021}
}

@article{ShrFu:J22,
  title={Deep spectrum cartography: Completing radio map tensors using learned neural models},
  author={Shrestha, Sagar and Fu, Xiao and Hong, Mingyi},
  journal={IEEE Trans. Signal Process.},
  volume={70},
  pages={1170--1184},
  year={2022}
}

@article{TimSub:J23,
  title={Quantized radio map estimation using tensor and deep generative models},
  author={Timilsina, Subash and Shrestha, Sagar and Fu, Xiao},
  journal={IEEE Trans. Signal Process.},
 volume={72},
  number={0},
  pages={173--189},
  year={2024}
}

@article{LeviRon:J21,
  title={RadioUNet: Fast radio map estimation with convolutional neural networks},
  author={Levie, Ron and Yapar, {\c{C}}a{\u{g}}kan and Kutyniok, Gitta and Caire, Giuseppe},
  journal={IEEE Trans. Wireless Commun.},
  volume={20},
  number={6},
  pages={4001--4015},
  year={2021}
}

@article{YapKan:J23,
  title={Real-time outdoor localization using radio maps: A deep learning approach},
  author={Yapar, {\c{C}}a{\u{g}}kan and Levie, Ron and Kutyniok, Gitta and Caire, Giuseppe},
  journal={IEEE Trans. Wireless Commun.},
  volume={22},
  number={12},
  pages={9703--9717},
  year={2023}
}

@inproceedings{AhaKal:C23,
  title={{5G} {NR} indoor positioning by joint {DL-TDoA} and {DL-AoD}},
  author={Ahadi, Mohsen and Kaltenberger, Florian},
  booktitle={IEEE Wirel. Commun. Netw. Conf. (WCNC)},
  pages={1--6},
  year={2023}
}

@article{XinChe:J24,
  title={Constructing Indoor Region-based Radio Map without Location Labels},
  author={Xing, Zheng and Chen, Junting},
  journal={IEEE Trans. Signal Process.},
volume ={72},
page={2512–2526},
  year={2024}
}

@inproceedings{XinChe:C24,
  title={{HMM}-based {CSI} Embedding for Trajectory Recovery from {RSS} Measurements of Non-Cooperative Devices},
  author={Xing, Zheng and Chen, Junting},
  booktitle={Proc. IEEE Int. Conf. Acoust. Speech Signal Process. (ICASSP)},
  pages={7060--7064},
  year={2024}
}

@article{SatSut:J21,
  title={Space-frequency-interpolated radio map},
  author={Sato, Koya and Suto, Katsuya and Inage, Kei and Adachi, Koichi and Fujii, Takeo},
  journal={IEEE Trans. Veh. Technol.},
  volume={70},
  number={1},
  pages={714--725},
  year={2021}
}

@article{ChoJeo:J22,
  title={Sensor-aided learning for {Wi-Fi} positioning with beacon channel state information},
  author={Choi, Jeongsik},
  journal={IEEE Trans. Wireless Commun.},
  volume={21},
  number={7},
  pages={5251--5264},
  year={2022}
}

@article{WanZhu:J24,
  title={Sparse {Bayesian} learning-based hierarchical construction for {3D} radio environment maps incorporating channel shadowing},
  author={Wang, Jie and Zhu, Qiuming and Lin, Zhipeng and Chen, Junting and Ding, Guoru and Wu, Qihui and Gu, Guochen and Gao, Qianhao},
  journal={IEEE Trans. Wireless Commun.},
volume={23},
  number={10},
  pages={14560--14574},
  year={2024}
}

@article{He:J18,
  title={Mobility model-based non-stationary mobile-to-mobile channel modeling},
  author={He, Ruisi and Ai, Bo and St{\"u}ber, Gordon L and Zhong, Zhangdui},
  journal={IEEE Trans. Wireless Commun.},
  volume={17},
  number={7},
  pages={4388--4400},
  year={2018}
}

@article{ZheChe:J25,
  title={Blind Construction of Angular Power Maps in Massive {MIMO} Networks},
  author={Xing, Zheng and Chen, Junting},
  journal={Submitted to IEEE Trans. Signal Process.},
  volume={},
  number={},
  pages={},
  year={2025}
}

@book{GraIzr:B14,
  title={Table of integrals, series, and products},
  author={Gradshteyn, Izrail Solomonovich and Ryzhik, Iosif Moiseevich},
  year={2014},
  publisher={Academic press}
}

@book{AbrMil:B65,
  title={Handbook of mathematical functions: with formulas, graphs, and mathematical tables},
  author={Abramowitz, Milton and Stegun, Irene A},
  volume={55},
  year={1965},
  publisher={Courier Corporation}
}

@article{OesCla:J03,
  title={A physical scattering model for {MIMO} macrocellular broadband wireless channels},
  author={Oestges, Claude and Erceg, Vinko and Paulraj, Arogyaswami J},
  journal={IEEE J. Sel. Areas Commun.},
  volume={21},
  number={5},
  pages={721--729},
  year={2003}
}

@article{AstDav:J02,
  title={The effects of local scattering on direction of arrival estimation with {MUSIC}},
  author={Astely, David and Ottersten, Bjorn},
  journal={IEEE Trans. Signal Process.},
  volume={47},
  number={12},
  pages={3220--3234},
  year={2002}
}

@article{ZhaQT:J95,
  title={Probability of resolution of the {MUSIC} algorithm},
  author={Zhang, QT},
  journal={IEEE Trans. Signal Process.},
  volume={43},
  number={4},
  pages={978--987},
  year={1995}
}

@article{GarWym:J18,
  title={Optimal precoders for tracking the {AoD} and {AoA} of a mmWave path},
  author={Garcia, Nil and Wymeersch, Henk and Slock, Dirk TM},
  journal={IEEE Trans. Signal Process.},
  volume={66},
  number={21},
  pages={5718--5729},
  year={2018}
}

@article{SybHo:J18,
  title={Solution and analysis of {TDOA} localization of a near or distant source in closed form},
  author={Sun, Yimao and Ho, KC and Wan, Qun},
  journal={IEEE Trans. Signal Process.},
  volume={67},
  number={2},
  pages={320--335},
  year={2018}
}

@article{WanUrr:J11,
  title={Weighted centroid localization algorithm: Theoretical analysis and distributed implementation},
  author={Wang, Jun and Urriza, Paulo and Han, Yuxing and Cabric, Danijela},
  journal={IEEE Trans. Wireless Commun.},
  volume={10},
  number={10},
  pages={3403--3413},
  year={2011}
}

\end{document}